\documentclass{article}

\usepackage[toc,page,header]{appendix}
\usepackage{minitoc}

\usepackage[final, nonatbib]{neurips_2022}
\usepackage[numbers]{natbib}
\usepackage{xr}
\usepackage[utf8]{inputenc} %
\usepackage[T1]{fontenc}    %
\usepackage{url}            %
\usepackage{booktabs}       %
\usepackage{amsfonts}       %
\usepackage{nicefrac}       %
\usepackage{microtype}      %
\usepackage{xcolor}         %
\usepackage{wrapfig}
\usepackage{float}
\usepackage{multirow}
\usepackage{adjustbox}
\definecolor{darkpastelgreen}{rgb}{0.01, 0.75, 0.24}
\usepackage[pagebackref=true,breaklinks=true,colorlinks,citecolor=darkpastelgreen,bookmarks=false]{hyperref}
\usepackage{authblk}

\usepackage{mypreamble}

\newcommand{\eg}{\textit{e}.\textit{g}. }
\usepackage{colortbl}
\newcommand{\highlight}[1]{{\cellcolor[gray]{0.8}#1}}
\usepackage{graphicx}
\usepackage{multirow}
\usepackage{algorithm, algorithmic}
\usepackage{graphicx}
\usepackage{soul}

\definecolor{darkgreen}{RGB}{30,160,30}

\usepackage{subcaption}
\title{Cost-Sensitive Self-Training for Optimizing Non-Decomposable Metrics}

\usepackage{authblk}

\author[1]{Harsh Rangwani\thanks{Equal Contribution. Code:  \href{https://github.com/val-iisc/CostSensitiveSelfTraining}{https://github.com/val-iisc/CostSensitiveSelfTraining}}\;}
\newcommand\CoAuthorMark{\footnotemark[\arabic{footnote}]}
\author[1]{Shrinivas Ramasubramanian\protect\CoAuthorMark\;}
\author[2]{Sho Takemori\protect\CoAuthorMark\;}
\author[2]{Takashi Kato} 
\author[2]{\authorcr Yuhei Umeda}
\author[1]{R. Venkatesh Babu}
\affil[1]{Video Analytics Lab, Indian Institute of Science, Bengaluru, India}
\affil[2]{Fujitsu Limited, Kanagawa, Japan}
\affil[ ]{\texttt{harshr@iisc.ac.in, shrinivas.ramasubramanian@gmail.com, takemori.sho@fujitsu.com, kato.takashi\_01@fujitsu.com, umeda.yuhei@fujitsu.com, venky@iisc.ac.in}}

\begin{document}

\doparttoc %
\faketableofcontents %

\part{} %

\maketitle

\begin{abstract}
Self-training based semi-supervised learning algorithms have enabled the learning of highly accurate deep neural networks, using only a fraction of labeled data. However, the majority of work on self-training has focused on the objective of improving accuracy whereas practical machine learning systems can have complex goals (e.g. maximizing the minimum of recall across classes etc.) that are non-decomposable in nature. In this work, we introduce the Cost-Sensitive Self-Training (\ttt{CSST}) framework which generalizes the self-training-based methods for optimizing non-decomposable metrics. We prove that our framework can better optimize the desired non-decomposable metric utilizing unlabeled data, under similar data distribution assumptions made for the analysis of self-training.  Using the proposed \ttt{CSST} framework we obtain practical self-training methods (for both vision and NLP tasks) for optimizing different non-decomposable metrics using deep neural networks.  Our results demonstrate that \ttt{CSST} achieves an improvement over the state-of-the-art in majority of the cases across datasets and objectives.
\end{abstract}

\section{Introduction}
In recent years, semi-supervised learning algorithms are increasingly being used for training deep neural networks~\cite{chapelle2009semi,kingma2014semi,sohn2020fixmatch,xie2020self}. These algorithms lead to accurate models by leveraging the unlabeled data in addition to the limited labeled data present. For example, it’s possible to obtain a model with minimal accuracy degradation ($\leq 1\%$) using 5\% of labeled data with semi-supervised algorithms compared to supervised models trained using 100\% labeled data~\cite{sohn2020fixmatch}. Hence, the development of these algorithms has resulted in a vast reduction in the requirement for expensive labeled data.

Self-training is one of the major paradigms for semi-supervised learning. It involves obtaining targets (\eg pseudo-labels) from a network from the unlabeled data, and using them to train the network further. The modern self-training methods also utilize additional regularizers that enforce prediction consistency across input transformations (e.g., adversarial perturbations~\cite{miyato2018virtual}, augmentations~\cite{xie2020unsupervised,sohn2020fixmatch}, etc.) , enabling them to achieve high performance using only a tiny fraction of labeled data. Currently, the enhanced variants of self-training with consistency regularization~\cite{zhang2021flexmatch,pham2021meta} are among the state-of-the-art (SOTA) methods for semi-supervised learning. 

Despite the popularity of self-training methods, most of the works~\cite{xie2020unsupervised, berthelot2019mixmatch, sohn2020fixmatch} have focused on the objective of improving prediction accuracy. However, there are nuanced objectives in real-world based on the application requirements. Examples include minimizing the worst-case recall~\cite{mohri2019agnostic} used for federated learning, classifier coverage for minority classes for ensuring fairness~\cite{goh2016satisfying}, etc. These objectives are complex and cannot be expressed just by using a loss function on the prediction of a single input (i.e., non-decomposable). There has been a considerable effort in optimizing non-decomposable objectives for different supervised machine learning models~\cite{narasimhan2021training,sanyal2018optimizing}. 
\begin{figure*}

\begin{minipage}[c]{0.55\linewidth}
    \includegraphics[width=\textwidth]{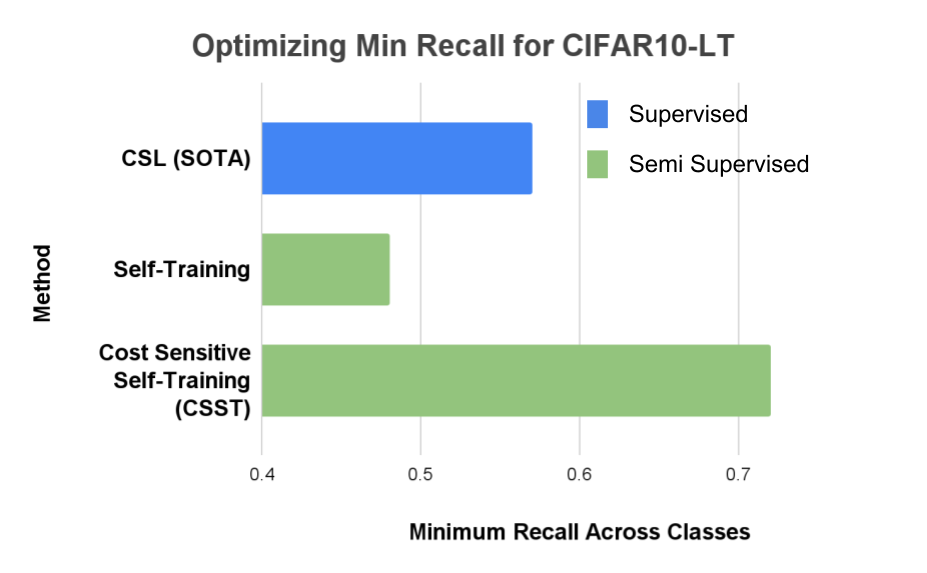}
\end{minipage}\hfill
\begin{minipage}[c]{0.45\linewidth}
  \caption{We show a comparison of the SOTA CSL~\cite{narasimhan2021training} method with the Self-training-based Semi-Supervised methods, for optimizing the minimum recall objective on the CIFAR10-LT dataset. Our proposed \ttt{CSST} framework produces significant gains in the desired metric leveraging additional unlabeled data through our proposed weighted novel consistency regularizer and thresholding mechanism.}
  \label{fig:overview}
\end{minipage}
\vspace{-5mm}
\end{figure*}
However, as supervision can be expensive, in this work we aim to answer the question:
\emph{Can we optimize non-decomposable objectives using self-training-based methods developed for semi-supervised learning?}  

 We first demonstrate that vanilla self-training methods (e.g., FixMatch~\cite{sohn2020fixmatch}, UDA~\cite{xie2020unsupervised}, etc.) can produce unsatisfactory results for non-decomposable metrics (Fig. \ref{fig:overview}). We then generalize the Cost-Sensitive Loss for Self-Training by introducing a novel weighted consistency regularizer, for a particular non-decomposable metric. Further, for training neural networks we introduce appropriate loss functions and pseudo label selection (thresholding) mechanisms considering the non-decomposable metric we aim to optimize. We also prove that we can achieve better performance on desired non-decomposable metric through our framework utilizing self-training, under similar assumptions on data distributions as made for theoretical analysis of self-training~\cite{wei2020theoretical}. We demonstrate the practical application by optimizing various non-decomposable metrics by plugging existing methods (\eg FixMatch~\cite{sohn2020fixmatch} etc.) into our framework. Our framework leads to a significant average improvement in desired metric of minimizing worst-case recall while maintaining similar accuracy (Fig. \ref{fig:overview}).

 In summary: \textbf{a)} we introduce a Cost-Sensitive Self-Training (\texttt{CSST})  framework for optimizing non-decomposable metrics that utilizes unlabeled data in addition to labeled data. (Sec.~\ref{sec:proposed-method})
 \textbf{b)} we provably demonstrate that our \texttt{CSST} framework can leverage unlabeled data to achieve better performance over baseline on desired non-decomposable metric (Sec.~\ref{sec:theoretical_results}) \textbf{c)} we show that by combining \ttt{CSST} with self-training frameworks (\eg FixMatch~\cite{sohn2020fixmatch}, UDA~\cite{xie2020unsupervised} etc.)  leads to effective optimization of non-decomposable metrics, resulting in significant improvement over vanilla baselines. (Sec.~\ref{sec:expt})

\section{Preliminaries}
\label{sec:preliminaries}
\subsection{Non-Decomposable Objectives and Reduction to Cost-Sensitive Learning}
\label{sec:ndo-loss-fun}
\begin{wraptable}[13]{r}{6.0cm}
\vspace{-6.8mm}
\caption{\addedtext{Metrics defined using entries of a confusion matrix.}}\label{wrap-tab:metrics}
\vspace{-5mm}
\addedtext{\begin{tabular}{cc}\\\midrule  
Metric & Definition  \\\midrule
\midrule
 Recall ($\mathrm{rec}_i[F]$) & $\frac{C_{i,i}[F]}{ \sum_j{C_{i,j}[F]} } $ \\  \midrule
Coverage  ($\mathrm{cov}_i[F]$) & $\sum_j{C_{j,i}[F]} $ \\   \midrule
Precision  ($\mathrm{prec}_i[F]$) & $\frac{C_{i,i}[F]}{ \sum_k{C_{k,i}[F]} } $ \\  \midrule
Worst Case Recall &  $ \min_{i} \frac{C_{i,i}[F]}{ \sum_j{C_{i,j}[F]} } $ \\  \midrule
Accuracy  & $\sum_{i}{C_{i,i}[F]}$ \\ \bottomrule 
\end{tabular}}
\vspace{1mm}
\end{wraptable} 

We consider the $K$-class classification problem
with an instance space $\mX$ and the set of labels $\mY = [K]$.
The data distribution on $\mX \times [K]$ is denoted by $D$.
For $i \in [K]$, we denote by $\pi_i$ the class prior $\prob(y = i)$. 
Notations commonly used across paper are in  Table \ref{tab:notations} present in Appendix.
For a classifier $F : \mX\rightarrow [K]$, 
we define confusion matrix $\mathbf{C}[F] \in \RR^{K \times K}$ by 
$C_{ij}[F] = \ex[(x, y) \sim D]{\indc(y=i, F(x)=j)}$.
Many metrics relevant to classification can be defined as functions of entries of confusion matrices such as class-coverage, recall and accuracy to name a few. We introduce more complex metrics,  which are of practical importance in the case of imbalanced distributions~\cite{cotter2019optimization} (Tab. \ref{wrap-tab:metrics}).

A classifier often tends to suffer low recalls on tail (\addedtext{minority}) classes in such cases. Therefore, one may want to maximize the worst case recall, \begin{equation*}
    \label{eq:min-recall-obj}
    \modifiedtext{\max_F \min_{i \in [K]} \text{rec}_i[F]}.
\end{equation*} 

Similarly, on long-tailed datasets, the tail classes suffer from low coverage, lower than their respective priors. An interesting objective in such circumstances is to maximise the mean recall, subject to the coverage being within a given margin.
\begin{equation}
   \label{eq:coverage-constraint-obj}
  \max_{F} \frac{1}{K}
  \sum_{i\in [K]}\text{rec}_i[F] \quad \text{s.t. } \cov_j[F] \ge \frac{0.95}{K}, \forall j \in [K].
\end{equation} 
Many of these metrics are \textbf{non-decomposable},
i.e., one cannot compute these metrics 
by simply calculating the average of scores on individual examples.
Optimizing for these metrics can be regarded as instances of cost-sensitive learning (CSL). More specifically, \emph{optimization problems of the form which can be written as a linear combination of $G_{i,j}$ and $C_{ij}[F]$ will be our focus in this work} where $\bG$ is a $K \times K$ matrix.
\begin{equation}
   \label{eq:csl-obj}
   \max_F \sum_{i, j \in [K]}G_{ij} C_{ij}[F],
\end{equation}
The entry $G_{ij}$ represents the reward associated with predicting class $j$ when the true class is $i$.
The matrix $\bG$ is called a gain matrix \cite{narasimhan2021training}.
Some more complex non-decomposable objectives for classification can be reduced to CSL \cite{narasimhan2015consistent,tavker2020consistent,narasimhan2021training}.
For instance, 
the aforementioned two complex objectives
can be reduced to CSL using continuous relaxation or a Lagrange multiplier as bellow.
  Let $\Delta_{K- 1} \subset \RR^{K}$ be the $K-1$-dimensional probability simplex.
Then, maximizing the minimum recall is equivalent to the 
saddle-point optimization problem:
\begin{equation}
   \max_F \min_{\blambda \in \Delta_{K-1}} \sum_{i \in [K]}\lambda_i \frac{C_{ii}[F]}{\pi_i}.
\end{equation}
Thus, for a fixed $\blambda$, 
the corresponding gain matrix is given as a diagonal matrix $\diag(G_1, \dots, G_K)$
with $G_i = \lambda_i / \pi_i$ for $1 \le i \le K$.
Similarly, using Lagrange multipliers $\blambda \in \RR_{\ge 0}^K$, 
Eq. \eqref{eq:coverage-constraint-obj} is rewritten as a max-min optimization problem \citep[Sec. 2]{narasimhan2021training}:
\begin{equation}
 \label{eq: cov-const-obj}
  \addedtext{ \max_F \min_{\blambda \in \RR^K_{\ge 0}} 
   \frac{1}{K}
   \sum_{i\in [K]} C_{ii}[F]/\pi_i +\sum_{j \in [K]}  \lambda_j \left(
      \sum_{i \in [K]} C_{ij}[F] - 0.95/K
   \right).}
\end{equation}
In this case,
the corresponding gain matrix $\mathbf{G}$ is given as $G_{ij} = \frac{\delta_{ij}}{K\pi_i} + \lambda_j$,
where $\delta_{ij}$ is the Kronecker's delta.
One can solve these max-min problems by alternatingly updating $\blambda$ 
(using exponented gradient or projected gradient descent) and 
optimizing the cost-sensitive objectives~\cite{narasimhan2021training}.

\subsection{Loss Functions for Non-Decomposable Objectives}
\label{sec:loss-fn-for-ndo}
\label{loss func: NDO}
The cross entropy loss function is appropriate for optimizing accuracy for deep neural networks, 
however, learning with CE can suffer low performance for cost-sensitive objectives \cite{narasimhan2021training}.
Following \cite{narasimhan2021training}, we introduce calibrated loss functions for given gain matrix $\mathbf{G}$.
We let 
$p_m: \mX \rightarrow \Delta_{K-1} \subset \RR^{K}$
be a prediction function of a model, where $\Delta_{K-1}$
is the $K-1$-dimensional probability simplex.
For a gain matrix $\mathbf{G}$, 
the corresponding loss function is 
given as a combination of logit adjustment \cite{menon2020long} and loss re-weighting \cite{patrini2017making}.
We decompose the gain matrix $\bG$ as $\bG = \bM \bD$, where 
$\bD = \diag(G_{11}, \dots, G_{KK})$ be a diagonal matrix, with $D_{ii} > 0, \forall i \in [K]$
and $\bM \in \RR^{K \times K}$.
For $y \in [K]$ and model prediction $p_m(x)$, the hybrid loss is defined as follows:
\begin{equation}
   \hybloss(y, p_m(x)) = -\sum_{i \in [K]} M_{yi} \log\left(
      \frac{\left(p_m(x)\right)_i/D_{ii}}{\sum_{j \in [K]} \left(p_m(x)\right)_j/D_{jj}}
   \right).   \label{loss:hyb}
\end{equation}
To make the dependence of $\bG$ explicit, we also denote $\hybloss(y, p_m(x))$ as $ \hybloss(y, p_m(x); \bG)$.
The average loss on training sample $S \subset \mX$ is defined as 
$\mathcal{L}^{\mathrm{hyb}}(\mX) = \frac{1}{| S |}\sum_{x \in S}{\ell^{\mathrm{hyb}}(u, p_m(x))}.$
\citet{narasimhan2021training} proved that the hybrid loss is calibrated,
that is learning with $\hybloss$ gives the Bayes optimal classifier for $\bG$ (c.f., \cite[Proposition 4]{narasimhan2021training}, 
\addedtext{of which we provide a formal statement in Sec. \ref{sec:appendix-formal-statement}}).
If $\bG$ is a diagonal matrix (i.e., $\bM = \bone_K$), 
the hybrid loss is called the logit adjusted (LA) loss
and $\hybloss(y, p_m(x))$ is denoted by $\laloss(y, p_m(x))$.

\subsection{Consistency Regularizer for Semi-Supervised Learning}
Modern self-training methods not only leverage pseudo labels, but also 
forces consistent predictions of a classifier on augmented examples or neighbor examples
\cite{wei2020theoretical,miyato2018virtual,xie2020unsupervised,sohn2020fixmatch}.
More formally, a classifier $F$ is trained so that the consistent regularizer $R(F)$ is small 
while a supervised loss or a loss between pseudo labeler are minimized \cite{wei2020theoretical,sohn2020fixmatch}.
Here the consistency regularizer $R(F)$ is defined as
\begin{equation*}
   \ex[x]{\indc(F(x) \neq F(x'), 
   \exists x' \text{ s.t. } x'\text{ is a neighbor of an  augmentation of } x)}.
\end{equation*}
In existing works, consistency regularizers are considered for optimization of accuracy.
In the subsequent sections,
we consider consistency regularizers for cost-sensitive objectives.
\section{Cost-Sensitive Self-Training \ for Non-Decomposable Metrics}
\label{sec:theoretical_results}
\subsection{CSL and Weighted Error}
\label{sec:csl-weighted-err}
In the case of accuracy or $\zo$-error,
a self-training based SSL algorithm using a consistency regularizer achieves the state-of-the-art performance
across a variety of datasets
\cite{sohn2020fixmatch} and 
its effectiveness has been proved theoretically \cite{wei2020theoretical}.
This section provides theoretical analysis of a self-training based SSL algorithm for non-decomposable objectives
by generalizing \cite{wei2020theoretical}.
More precisely, 
the main result of this section (Theorem \ref{thm:main-err-bd}) states that
an SSL method using consistency regularizer improves a given pseudo labeler for non-decomposable objectives.
We provide all the omitted proofs in Appendix for theoretical results in the paper.

In Sec. \ref{sec:preliminaries}, we considered non-decomposable metrics and their reduction to
cost-sensitive learning objectives defined by Eq. \eqref{eq:csl-obj} using a gain matrix.
In this section, we consider an equivalent objective using the notion of weighted error.
For weight matrix $w = (w_{ij})_{1 \le i, j\le K}$ 
and a classifier $F : \mX \rightarrow [K]$,
a weighted error is defined as follows:
\begin{align*}
\err_w(F) = \sum_{i, j \in [K]}w_{ij} \ex[x \sim P_i]{\indc (F(x) \ne j)},
\end{align*}
where, $P_i(x)$ denotes the class conditional distribution $\prob(x \mid y = i)$.
If $w = \diag(1/K, \dots, 1/K)$, then this coincides with the usual balanced error \cite{menon2020long}.
Since
\begin{math}
    C_{ij}[F] = \ex[(x, y) \sim D]{\indc(y = i, F(x) = j)}
    = \prob(y = i) - \prob(y = i)\ex[x \sim P_i]{\indc(F(x) \ne j)},
\end{math} we can write:
\begin{equation*}
\scriptsize
    G_{ij} C_{ij}[F] = G_{ij} (\prob(y = i) - \prob(y = i)\ex[x \sim P_i]{\indc(F(x) \ne j)}) = G_{ij} (\pi_{i} - \pi_{i} \ex[x \sim P_i]{\indc (F(x) \ne j)})
\end{equation*}
Here $\pi_i$ is the class prior $\prob(y = i)$ for $1\le i \le K$ as before. Hence  CSL objective (\ref{eq:csl-obj}) i.e. $\max_{F} \sum_{i,j} G_{ij} C_{ij}[F]$ is equivalent to minimizing the negative term above i.e.
$ G_{ij} \pi_{i} \ex[x \sim P_i]{\indc (F(x) \ne j)}$ which is same as $\err_{w}(F)$ with 
$w_{ij} = G_{ij} \pi_i$ for $1 \le i, j \le K$. Hence, the \emph{notion of weighted error is equivalent to CSL}, which we will also use later for deriving loss functions. We further note that if we add a matrix with the same columns ($c\indc \geq 0$)  to the gain matrix $\bG$, still
the maximizers of CSL \eqref{eq:csl-obj} are the same as the original problem.
Hence, without loss of generality, we assume $w_{ij}\ge 0$.
We assume $w \ne \bm{0}$, i.e., $|w|_1 > 0$ for avoiding degenerate solutions.

In the previous work \cite{wei2020theoretical}, it is assumed that 
there exists a ground truth classifier $F^\star: \mX \rightarrow [K]$
and the supports of distributions $\{P_i\}_{1 \le i \le K}$ are disjoint. 
However, if supports \addedtext{of distributions $\{P_i\}_{1 \le i \le K}$} are disjoint, a solution of the minimization problem $\min_F \err_w(F)$
is independent of $w$ in some cases.
More precisely, if $w = \diag(w_1, \dots, w_K)$ i.e. a diagonal matrix and $w_i > 0, \forall i$, 
then the optimal classifier is given as $x \mapsto \argmax_{k\in[K]} w_k P_k(x)$ 
(this follows from \cite[Proposition 1]{narasimhan2021training}).
If supports are disjoint, 
then the optimal classifier is the same as $x \mapsto \argmax_{k \in [K]} P_k(x)$, which coincides with the 
ground truth classifier.
Therefore, we do not assume the supports of $P_i$ are disjoint nor a ground truth classifier exists unlike \cite{wei2020theoretical}.
See Fig. \ref{fig:disjoint_non_disjoint_supp} for an intuitive explanation.
\begin{figure}[htb]
    \centering 
    \includegraphics[width=\linewidth]{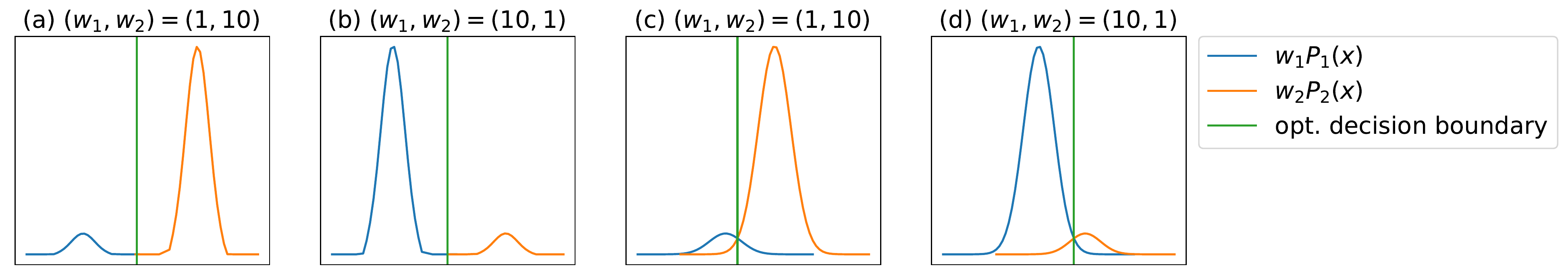}
    \caption{
    Using a simple example, we explain a difference in theoretical assumptions compared to 
    \cite{wei2020theoretical} that assumes $\{P_i\}_{1\le i \le K}$ have disjoint supports.
    Here, we consider the case when $K=2$, $w = \diag(w_1, w_2)$, and 
    $P_1$ and $P_2$ are distributions on $\mX \subset \RR$.
            (a), (b): In a perfect setting where two distributions $P_1$ and $P_2$ have disjoint supports, 
the Bayes optimal classifier for the CSL is identical to the ground truth classifier 
($x \mapsto \argmax_i P_i(x)$) for any choices of weights $(w_1, w_2)$.
(c), (d): 
In a more generalized setting, 
the Bayes optimal classifier $x \mapsto \argmax_{i}w_i P_i(x)$ depends on the choice of weights (i.e., gain matrix). 
The optimal decision boundary (the green line) for the CSL moves to the right
as $w_1/w_2$ increases.
    }
    \label{fig:disjoint_non_disjoint_supp}
\end{figure}

\subsection{Weighted Consistency Regularizer}
\label{subsec:weighted-consistency-def}
For improving accuracy, the consistency in prediction is equally important across the distributions $P_i$
for $1 \le i \le K$ \cite{sohn2020fixmatch,wei2020theoretical}.
However, for our case of weighted error, 
if the entries of the $i_0^{th}$ row of the weight matrix $w$ are larger
than the other entries for some $i_0$,
then the consistency of model predictions on examples drawn from the distribution $P_{i_0}$ are more important 
than those on the other examples. 
In this case, we require more restrictive consistency regularizer 
for distribution for $P_ {i_0}$.
Thus, we require a weighted (or cost-sensitive) consistency regularizer, which we define below.

We assume that the instance space $\mX$ is a normed vector space with norm $|\cdot|$
and $\mT$ is a set of augmentations, i.e., each $T \in \mT$ is a map from $\mX$ to itself.
For a fixed $r > 0$, we define $\mB(x)$ by 
$\{x' \in \mX : \exists T \in \mT \text{ s.t. } | x' - T(x)| \le r\}$.
For each $i \in [K]$, we define conditional consistency regularizer by 
\begin{math}
    R_{\mB, i}(F) = \ex[x\sim P_i]{\indc\left(\exists x' \in \mB(x) \text{ s.t. } F(x) \ne F(x') \right)}.
\end{math}
Then, we define the weighted consistency regularizer by  %
\begin{math}
    R_{\mB, w}(F) = \sum_{i, j \in [K]}w_{ij} R_{\mB, i}(F).
\end{math}
\addedtext{If the prediction of classifier $F$ is robust to augmentation $T \in \mT$ and small noise, then $R_{\mB, w}(F)$ is small.}
For $\regub > 0$, we consider the following optimization objective for finding a classifier $F$:
\begin{equation}
    \label{eq:cons-cnst-obj}
    \min_{F}\err_w(F) \quad \text{subject to } R_{\mB, w}(F) \le \regub.
\end{equation}
A solution of the problem \eqref{eq:cons-cnst-obj} is denoted by $\gstar$.
\addedtext{We let $\gpl: \mX \rightarrow [K]$ a pseudo labeler (a classifier) with reasonable performance (we elaborate on this in Section \ref{subsec:weighted_consistency}).}
The following mathematically informal assumption below is required to interpret our main theorem.
\begin{assumption}
    \label{assump:mu-small}
    We assume both $\regub$ and $\err_w(\gstar)$ are sufficiently small so that they are 
    negligible compared to $\err_w(\gpl)$.
\end{assumption}
\addedtext{Assumption 1 assumes existence of an optimal classifier $F^*$ that  minimizes the error $\err_w(F)$ (i.e. Bayes Optimal) among the class of classifiers which are robust to data augmentation (i.e. low weighted consistency $R_{\mB, w}(F)$). As we operate in case of overparameterized neural networks such a classifier $F^*$ is bound to exist, but is unknown in our problem setup.}
In the case of the balanced error, 
the validity of this assumption is justified by the fact that 
the existing work \cite{sohn2020fixmatch} using consistency regularizer on data augmentation obtains classifier $F$ , that achieves 
high accuracy (i.e., low balanced errors) for balanced datasets.
Also in Appendix \ref{sec:appendix-examples-for-theoretical}, we provide an example that supports the validity of the assumption in the case of 
Gaussian mixtures and diagonal weight matrices.

\subsection{Expansion Property}
For $x \in \mX$, we define the neighborhood $\mN(x)$ of $x$ by $\{x' \in \mX: \mB(x) \cap \mB(x') \ne \emptyset\}$.
For a subset $S \subseteq \mX$, neighborhood of $S$ is defined as $\mN(S) = \cup_{x \in S} \mN(x)$.
Similarly to \cite{wei2020theoretical}, we consider the following property on distributions.
\begin{definition}
    \label{def:c-exp-def}
    \addedtext{Let $c: (0, 1] \rightarrow [1, \infty)$ be a non-increasing function.}
    For a distribution $Q$ on $\mX$ %
    we say $Q$ has $c$-expansion property if
    $Q(\mN(S)) \ge c(Q(S)) Q(S)$ for any measurable $S \subseteq \mX$.
\end{definition}

The $c$-expansion property implies that if $Q(S)$ decreases, then the 
``expansion factor'' ${Q(\mN(S))}/{Q(S)}$ increases. 
This is a natural condition, because it roughly requires that 
if $Q(S)$ is small, then $Q(\mN(S))$ is large compared to $Q(S)$. \addedtext{For intuition 
let us consider a ball of radius $l$ depicting $S \subset \RR^d$ with volume $Q(S)$ 
and its neighborhood $\mN(S)$ expands to a ball with radius $l+1$. 
The expansion factor here would be $((l + 1)/l)^d$, hence as $l$  (i.e. $Q(S)$) 
increases $(1 + 1/l)^d$ (i.e. ${Q(\mN(S))}/{Q(S)}$) decreases. 
Hence, it's natural to expect $c$ to be a non-decreasing function.}
\modifiedtext{The $c$-expansion property (on each $P_i$) considered here is equivalent to the $(a, \widetilde{c})$-expansion property,
which is shown to be realistic for vision and used for 
theoretical analysis of self-training in \cite{wei2020theoretical}, where $a \in (0, 1)$ and $\widetilde{c} > 1$
(see Sec. \ref{sec:appendix-a-c-expansion} in Appendix).} 
In addition, we also show that it is also satisfied for mixtures of Gaussians and 
mixtures of manifolds (see Example \ref{exa:c-exp-mix-gauss} in Appendix for more details). 
Thus, the $c$-expansion property is a general property satisfied for a wide class of distributions.
\begin{assumption}
    \label{assump:c-exp}
    For weighted probability measure $\Pw$ on $\mX$ by 
    \begin{math}
        \Pw(U) = 
        \frac{\sum_{i, j \in [K]}w_{ij}P_i(U)}{\sum_{i,j \in [K]} w_{ij}}
    \end{math}
    for $U \subseteq \mX$.
    We assume $\Pw$ satisfies $c$-expansion
    for a non-increasing function $c: (0, 1] \rightarrow [1, \infty)$.
\end{assumption}

\subsection{Cost-Sensitive Self-Training with Weighted Consistency Regularizer}
In this section we first introduce the assumptions on the pseudo labeler $\gpl$ and then introduce the theoretical Cost-Sensitive Self-Training (\texttt{CSST}) objective.
\label{subsec:weighted_consistency}
$\gpl$ can be any classifier satisfying the following assumption, however, typically it is a classifier trained on a labeled dataset.
\begin{assumption}
    \label{assump:pseudo-labeler}
    We assume that $\err_w(\gpl) + \err_w(\gstar) \le |w|_1$.
    Let $\gamma = c(p_w)$, where 
    \begin{math}
    p_w = \frac{\err_w(\gpl) + \err_w(\gstar)}{|w|_1}. 
    \end{math}
    We also assume $\gamma > 3$.
\end{assumption}
Since $c$ is non-increasing, $\gamma$ (as a function of $\err_w(\gpl)$) is a non-increasing function of $\err_w(\gpl)$ (and $\err_w(\gstar)$).
Therefore, the assumption $\gamma > 3$ roughly requires that $\err_w(\gpl)$ is ``small''.
We provide concrete conditions for $\err_w(\gpl)$ that satisfy $\gamma > 3$ 
in the case of mixture of isotropic $d$-dimenional Gaussians for a region $\mB(x)$ defined by $r$ in 
Appendix \addedtext{ (Example \ref{exa:gamma-3})}.
In the example, we show that 
the condition $\gamma > 3$ 
is satisfied if $\err_w(\gpl) < 0.17$
in the case when $r = 1/(2\sqrt{d})$
and satisfied if $\err_w(\gpl) < 0.33$ in the case when $r = 3/(2\sqrt{d})$, where $\mX \subseteq \RR^d$.
Since we assume $\err_w(\gstar)$ is negligible compared to $\err_w(\gpl)$ (Assumption \ref{assump:mu-small}),
the former condition in Assumption \ref{assump:pseudo-labeler} is approximately equivalent to 
$\err_w(\gpl) \le |w|_1$ which is satisfied by the definition of $\err_w$.

We define $L^{(i)}_{\zo}(F, F') = \ex[x \sim P_i]{\indc(F(x) \ne F'(x))}$.
Then, we consider the following objective:
\begin{equation}
    \label{eq:ssl-non-decomp-obj-theoretical}
    \min_F \mL_w(F), \quad 
    \text{where } \mL_w(F) =\frac{\gamma + 1}{\gamma - 1}L_w(F, \gpl) + \frac{2\gamma}{\gamma - 1}R_{\mB, w}(F).
\end{equation}
Here $L_w(F, \gpl)$ is defined as 
\begin{math}
    \sum_{i, j \in [K]}w_{ij} L^{(i)}_\zo(F, \gpl).
\end{math}
The above objective corresponds to cost-sensitive self-training (with $\gpl$) objective with weighted consistency regularization. We provide following theorem which relates the weighted error of classifier $\hat{F}$ learnt using the above objective to the weighted error of the pseudo labeler ($\gpl$).
\begin{theorem}
    \label{thm:main-err-bd}
    \modifiedtext{Any learnt classifier $\widehat{F}$ using the loss function $\mL_w$
    (i.e.,  $\argmin_F\mL_w(F)$)} satisfies:
    \begin{align*}
        \err_w(\widehat{F}) 
        &\le
        \frac{2}{\gamma - 1}
        \err_w(\gpl)
        + \frac{\gamma + 1}{\gamma - 1} \err_w(\gstar)
        + \frac{4\gamma}{\gamma - 1} R_{\mB, w}(\gstar).
    \end{align*}
\end{theorem}
\begin{remark}
    Since both $ \err_w(\gstar)$ and $R_{\mB, w}(\gstar) \le \regub$ are negligible compared to $\err_w(\gpl)$ 
    and $\gamma > 3$,
    Theorem~\ref{thm:main-err-bd} asserts that \addedtext{$\hat{F}$} learnt 
    by minimizing semi-supervised loss $L_w(F, \gpl)$ with the consistency regularizer $R_{\mB, w}(F)$ can 
    achieve superior performance than the pseudo labeler \addedtext{$\gpl$}
    in terms of the weighted error $\err_w$. The above theorem  is a generalization of \cite[Theorem 4.3]{wei2020theoretical}, which provided
a similar result for balanced $\zo$-error in the case of distributions with disjoint supports.
    In Appendix \addedtext{Sec. \ref{sec:appendix-all-layer}}, following \cite{wei2020theoretical,wei2019improved}, we also provide a generalization bound for $\err_w(F)$ using all-layer margin \cite{wei2019improved}
    in the case when classifiers are neural networks.
\end{remark}

\section{\texttt{CSST} in Practice}

In the previous section, we proved that by using self-training (\texttt{CSST}), we can achieve a superior classifier $\hat{F}$ in comparison to pseudo labeler $\gpl$ through weighted consistency regularization. As we have established the equivalence of the weighted error $Err_{w}$ to the CSL objective expressed in terms of $\bG$ (Sec. \ref{sec:csl-weighted-err}) , we can theoretically optimize a given non-decomposable metric expressed by $\bG$ better using \ttt{CSST}, utilizing the additional unlabeled data via self-training and weighted consistency regularization.  We now show how \ttt{CSST} can be used in practice for optimizing non-decomposable metrics in the case of neural networks.

The practical self-training methods utilizing consistency regularization (e.g., FixMatch~\cite{sohn2020fixmatch}, etc.) for semi-supervised learning have supervised loss $\mathcal{L}_s$ for labeled and consistency regularization loss for unlabeled samples (i.e., $\mathcal{L}_u$) with a thresholding mechanism to select unlabeled samples. The final loss for training the network is $\mathcal{L}_s + \lambda_{u}\mathcal{L}_u$, where $\lambda_{u}$ is the hyperparameter. The supervised loss $\mathcal{L}_s$ can be modified conveniently based on the desired non-decomposable metric by using suitable $\bG$ (Sec. \ref{sec:ndo-loss-fun}). We will now introduce the novel weighted consistency loss and its corresponding thresholding mechanism for unlabeled data in \ttt{CSST}, used for optimizing desired non-decomposable metric.   

\vspace{1mm} \noindent \textbf{Weighted Consistency Regularization.}
As the idea of consistency regularization is to enforce consistency between model prediction on different augmentations of input, this is usually achieved by minimizing some kind of divergence $\mathcal{D}$. A lot of recent works~\cite{miyato2018virtual, sohn2020fixmatch,xie2020unsupervised} in semi-supervised learning  use $\mathcal{D}_{\mathrm{KL}}$ to enforce consistency between the model's prediction on unlabeled data and its augmentations, $p_{m}(x)$ and $p_{m}(\mathcal{A}(x))$. Here $\mathcal{A}$ usually denotes a form of strong augmentation. Across these works, the distribution of confidence of the model's prediction is either sharpened or used to get a hard pseudo label to obtain $\hat{p}_{m}(x)$. As we aim to optimized the cost-sensitive learning objective, we aim to match the distribution of normalized distribution (i.e. $\norm(\bG^{\mathbf{T}}\hat{p}_{m}(x)) = \bG^{\mathbf{T}}\hat{p}_{m}(x)/\sum_i{(\bG^{\mathbf{T}}\hat{p}_{m}(x))_i}$) (Proposition 2~\cite{narasimhan2021training} also in Prop.  \ref{prop:prop2-narashimhan} ) with the $p_{m}(\mathcal{A}(x))$ by minimizing the KL-Divergence between these. We now propose to use the following weighted consistency regularizer loss function for optimizing the same:
\begin{equation}
  \label{eq:wt-consistency-reg}
    \ell^{\mathrm{wt}}_{u}(\hat{p}_{m}(x), p_{m}(\mathcal{A}(x)), \bG) = -\sum_{i=1}^{K}(\bG^{\mathbf{T}}\hat{p}_{m}(x))_i \log(p_{m}(\mathcal{A}(x))_i)
\end{equation}
\begin{proposition}
The minimizer of $\mathcal{L}^{wt}_{u} =\frac{1}{|B_u|}\sum_{x \in B_u} \ell^{wt}_{u}(\hat{p}_{m}(x), p_{m}(\mathcal{A}(x)), \bG)$ 
leads to minimization of KL Divergence i.e. $\mathcal{D}_{KL}(\norm(\bG^{\mathbf{T}}\hat{p}_{m}(x)) || p_{m}(\mathcal{A}(x))) \; \forall x \in \mX $  .
\label{prop:kl-weighted}
\end{proposition}
As the above loss is similar in nature to the cost sensitive losses introduced by~\citet{narasimhan2021training} (Sec. \ref{sec:ndo-loss-fun}) we can use the logit-adjusted variants (i.e. $\ell^{\mathrm{LA}}$ and $\ell^{\mathrm{hyb}}$ based on type of $\bm{G}$) of these in our final loss formulations ($\mathcal{L}^{wt}_{u}$) for training overparameterized deep networks.  We further show in Appendix Sec. \ref{sec:appendix-conn-between} that the above loss $\ell^{\mathrm{wt}}_{u}$ approximately minimizes the theoretical weighted consistency regularization term $R_{\mB, w}(F)$ defined in Sec. \ref{subsec:weighted_consistency}.

\vspace{1mm} \noindent \textbf{Threshold Mechanism for \ttt{CSST}.}
\label{subsec:threshold}
In the usual semi-supervised learning formulation~\cite{sohn2020fixmatch} we use the confidence threshold ($\max_{i}(p_{m}(x)_i) > \tau$) as the function to select samples for which consistency regularization term is non-zero. We find that this leads to sub-optimal results in particularly the case of non-diagonal $\mathbf{G}$ as only a few samples cross the threshold (Fig. \ref{fig:mask_rate}).  As in the case of cost-sensitive loss formulation the samples may not achieve the high confidence to cross the threshold of consistency regularization. This is also theoretically justified by the following proposition:  
\begin{proposition}[\cite{narasimhan2021training} Proposition 2]
   \label{prop:prop2-narashimhan}
   Let $p_{m}^{opt}(x)$ be the optimal softmax model function obtained by optimizing the cost-sensitive objective in Eq. \eqref{eq:csl-obj} by averaging weighted loss function $\ell^{wt}(y,p_m(x)) = - \sum_{i=1}^{K} G_{y,i} \log{\frac{(p_m(x)_i)}{\sum_{j}p_m(x)_j}}$. Then optimal $p_{m}^{opt}(x)$ is: $p_{m}^{opt}(x) = \frac{G_{y,i}}{\sum_{j} G_{y,j}} = \norm(\bG^{T}\mathbf{y})  \forall (x,y)$.
\end{proposition}
Here $\mathbf{y}$ is the one-hot representation vector for a label $y$. This proposition demonstrates that for a particular sample the high confidence $p_{m}(x)$ may not be optimal based on $\bG$. We now propose our novel way of thresholding samples for which consistency regularization is applied in \texttt{CSST}. Our thresholding method takes into account the objective of optimizing the non-decomposable metric by taking $\bG$ into account. We propose to use the threshold on KL-Divergence of the softmax of the sample $ {p}_{m}(x)$ with the optimal softmax (i.e. $\norm(\bG^{T}\hat{{p}}_{m}(x))$) for a given $\mathbf{G}$ corresponding to the pseudo label (or sharpened) $\hat{{p}}_{m}(x)$, using which we modify the consistency regularization loss term:
\begin{equation}
    \label{loss:CSST-KL}
    \mathcal{L}_u^{wt}(B_u) = \frac{1}{|B_u|}\sum_{x \in B_u}\indc_{(\mathcal{D}_{KL} (\norm(\bG^{T}\hat{p}_{m}(x))  \; || \; p_m(x))\leq \tau)}
     \ell^{\mathrm{wt}}_{u}(\hat{p}_{m}(x), p_{m}(\mathcal{A}(x)), \bG)
\end{equation}
We name this proposed combination of KL-Thresholding and weighted consistency regularization as $\texttt{CSST}$ in our experimental results. We find that for non-diagonal gain matrix $\bG$ the proposed thresholding plays a major role in improving performance over supervised learning. This is demonstrated by comparison of $\ttt{CSST}$ and \ttt{CSST} w/o KL-Thresholding (without proposed thresholding mechanism) in Fig. \ref{fig:mask_rate} and Tab. \ref{tab:coverage}. We will now empirically incorporate \ttt{CSST} by introducing consistency based losses and thresholding mechanism for unlabeled data, into the popular semi-supervised methods of FixMatch~\cite{sohn2020fixmatch} and Unsupervised Data Augmentation for Consistency Training (UDA)~\cite{xie2020unsupervised}. The exact expression for the weighted consistency losses utilized for UDA and FixMatch have been provided in the Appendix \addedtext{ Sec. \ref{sec:appendix-fixmatch-obj} and \ref{sec:appendix-uda-obj}} .

\label{sec:proposed-method}

\section{Experiments}
\label{sec:expt}

\begin{figure*}[!t]
  \centering
\begin{minipage}[c]{0.68\textwidth}
    \centering
    \vspace{-5mm}
    \captionof{table}{Results of maximizing the worst-case recall over 
        \emph{all classes} (col 2--3) and
        over just the head and tail classes (col 4--7). }
    \vspace{2mm}
    \label{tab:min-recall}
    \centering
        \resizebox{\textwidth}{!}{
            \begin{tabular}{ccccccccc}
            \hline
            \multicolumn{1}{c}{\textbf{Method}} & 
            \multicolumn{2}{c}{\textbf{CIFAR10-LT ($\rho = 100$)}} & 
            \multicolumn{2}{c}{\textbf{CIFAR100-LT ($\rho = 10$)}} &
            \multicolumn{2}{c}{\textbf{Imagenet100-LT ($\rho = 10$)}} &
            \\
            & \textbf{Avg. Rec} & \textbf{Min. Rec} & 
            \textbf{Avg. Rec } & \textbf{Min. HT Rec} &
            \textbf{Avg. Rec } & \textbf{Min. HT Rec} &
            \\
            \hline
            \texttt{ERM} &  
                0.52 &	0.26 & 
                0.36 &	0.14 & 0.40 & 0.30
                \\
            \texttt{LA} &  
                0.51 &	0.38 & 
                0.36 &	0.35 & 0.48 & 0.47
                \\
            \texttt{CSL} &  
                0.64 &	0.57 & 
                0.43 &	0.43 & 0.52 & 0.52
                \\
            \hline
            \begin{tabular}{@{}c@{}}\texttt{Vanilla} \\ (FixMatch)\end{tabular}    
            
                & \highlight{0.78} &	0.48 &
                \highlight{0.63} &	0.36 & 0.58 & 0.49
                \\
            \begin{tabular}{@{}c@{}}\texttt{CSST} \\ (FixMatch)\end{tabular}
                & 0.76 &	\highlight{0.72} &
                \highlight{0.63} &	\highlight{0.61} & \highlight{0.64} & \highlight{0.63}
                \\
            \hline
        
        \end{tabular}}
  \end{minipage}  
\begin{minipage}[c]{0.3\linewidth}
    \centering
    \includegraphics[width=\linewidth]{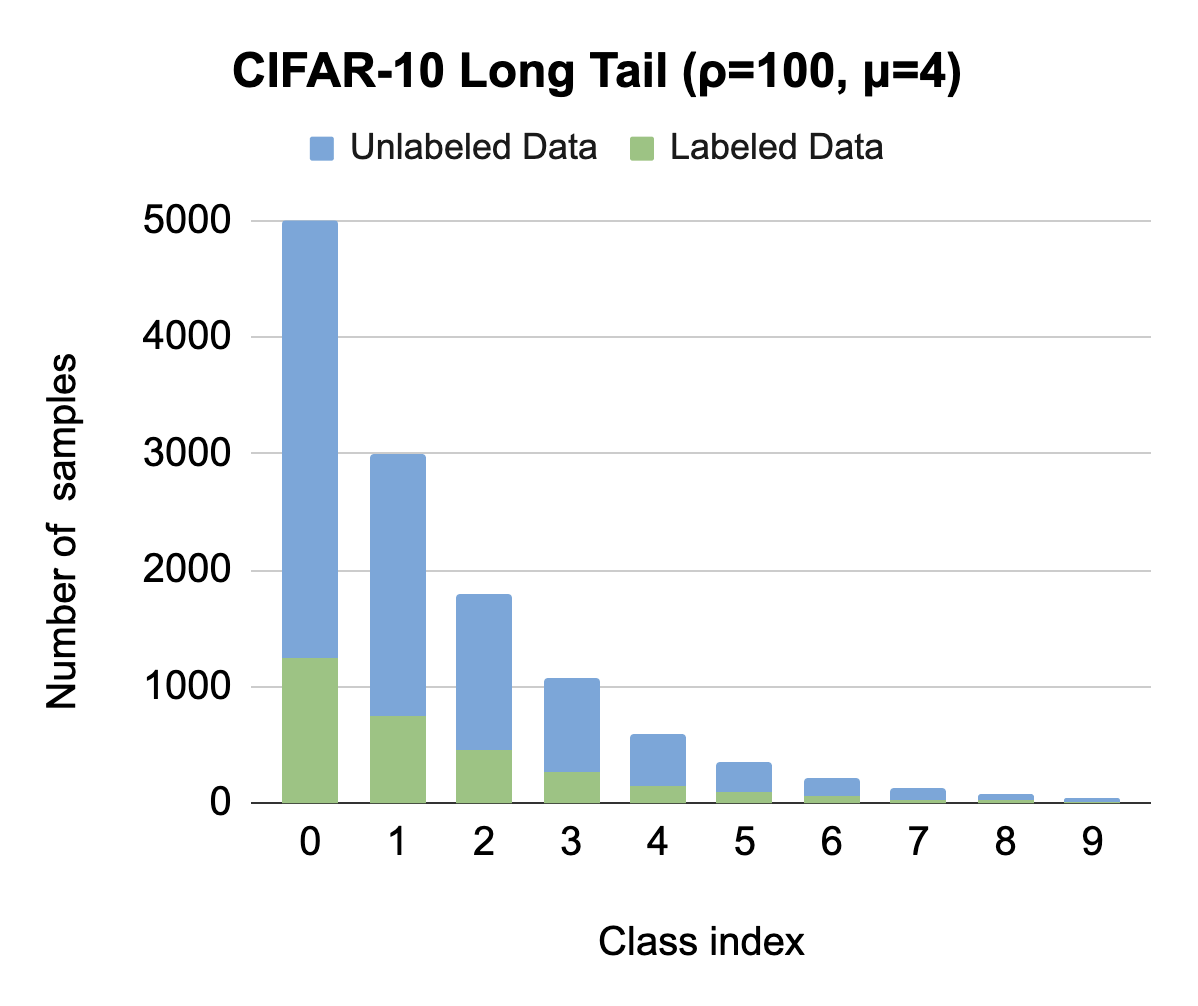}
    \vspace{-5mm}
    \captionof{figure}{CIFAR-10 Long tail distribution $\rho=100,\mu=4$.}
    \label{fig:intra-fid}
  \end{minipage}

  \vspace{-5mm}
\end{figure*}
We demonstrate that the proposed \ttt{CSST} framework shows significant gains in performance on both vision and NLP tasks on imbalanced datasets, with an imbalance ratio defined on the training set as $\rho = \frac{\max_i P(y=i)}{\min_i P(y=i)}$. We assume the labeled and unlabeled samples come from a similar data distribution and the unlabeled samples are much more abundant ($\mu$ times) the labeled. The frequency of samples follows an exponentially decaying long-tailed distributed as seen in  Fig.~\ref{fig:intra-fid}, which closely imitates the distribution of real-world data~\cite{thomee2016yfcc100m, krishna2017visual}.
 For CIFAR-10~\cite{krizhevsky2009learning}, IMDb~\cite{maas-EtAl:2011:ACL-HLT2011} and DBpedia-14~\cite{lehmann2015dbpedia}, we use $\rho = 100$ and $\rho = 10$ for CIFAR-100~\cite{krizhevsky2009learning} and ImageNet-100~\cite{russakovsky2015imagenet}~\footnote{https://www.kaggle.com/datasets/ambityga/imagenet100} datasets.
We compare our method against supervised methods of ERM, Logit Adjustment (\texttt{LA})~\cite{menon2020long} and Cost Sensitive Learning (\texttt{CSL})~\cite{narasimhan2014statistical} trained on the same number of labeled samples as used by semi-supervised learning methods, along with vanilla semi-supervised methods of FixMatch (for vision) and UDA (for NLP tasks). We use WideResNets(WRN) \cite{BMVC2016_87}, specifically WRN-28-2 and WRN-28-8 for CIFAR-10 and CIFAR-100 respectively. For ImageNet, we use a ResNet-50~\cite{he2016deep} network for our experiments and finetuned DistilBERT(base uncased)~\cite{sanh2019distilbert} for IMDb and DBpedia-14 datasets. We divide the balanced held-out set for each dataset equally into validation and test sets.  A detailed list of hyper-parameters and additional experiments can be found in the Appendix \addedtext{Tab. \ref{tab:hyperparams}} and Sec. \ref{sec: additional-experiments} respectively.

\vspace{1mm} \noindent \textbf{Maximizing Worst-Case Recall.}
\label{subsec:worst-case-recall}
For CIFAR-10, IMDb, and DBpedia-14 datasets, we maximize the minimum recall among all classes (Eq. \eqref{eq:min-recall-obj}).
Given the low number of samples per class for datasets with larger number of classes like CIFAR-100 and ImageNet-100, we pick objective \eqref{eq:min-HT-recall-obj}. We define the head classes ($\mH$) and tail classes ($\mT$) as the first 90 classes and last 10 classes respectively. 
The Min. HT recall objective can be mathematically formulated as: 
\begin{equation}
    \label{eq:min-HT-recall-obj}
  \max_F \min_{(\lambda_{\mathcal{H}}, \lambda_{\mathcal{T}}) \in \Delta_{1}} \frac{\lambda_{\mathcal{H}}}{|\mathcal{H}|} \sum_{i \in \mathcal{H}}\frac{C_{ii}[F]}{\pi_i} + \frac{\lambda_{\mathcal{T}}}{|\mathcal{T}|} \sum_{i \in \mathcal{T}}\frac{C_{ii}[F]}{\pi_i}.
\end{equation}
The corresponding gain matrix $\bf{G}$ is diag($\frac{\lambda_{\mH}}{\pi_1 |\mH|}, \frac{\lambda_{\mH}}{\pi_2 |\mH|}, \dots , \frac{\lambda_{\mT}}{\pi_{K-1} |\mT|}, \frac{\lambda_{\mT}}{\pi_{K} |\mT|}$). 
Since $\bf{G}$ is diagonal here, we use \texttt{CSST}(FixMatch) loss function Eq. \eqref{loss:CSST-KL} with the corresponding $\ell_{u}^{\mathrm{wt}}$ being substituted by $\ell_{u}^{\mathrm{LA}}$ as define in Sec. \ref{sec:ndo-loss-fun}. Also for labeled samples we use $\mathcal{L}_{s}^{\mathrm{LA}}$ as $\bG$ is diagonal, we then combine the loss and train network using SGD. Each few steps of SGD, were followed by an update on the $\blambda$ and $\bG$ based on the uniform validation set (See Alg. \textcolor{red}{1} in Appendix).  We find that \texttt{CSST}(FixMatch) significantly outperforms the other baselines in terms of the  Min. recall and Min. Head-Tail recall for all datasets, the metrics which we aimed to optimize (Tab. \ref{tab:min-recall}), which shows effectiveness of \ttt{CSST} framework. 
Despite optimizing worst-case recall we find that our framework is still able to maintain reasonable average (Avg.) recall in comparison to baseline \ttt{vanilla}(FixMatch), which demonstrates it's practical applicability. We find that optimizing Min. recall across NLP tasks of classification on long-tailed data by plugging UDA into  \ttt{CSST(UDA)} framework shows similar improvement in performance (Tab. \ref{tab:nlp_results}). This establishes the generality of our framework to even self-training methods across domain of NLP as well. 
As the $\mathbf{G}$ is a diagonal matrix for this objective, the proposed KL-Based Thresholding here is equivalent to the confidence based threshold of FixMatch in this case. Despite the equivalence of thresholding mechanism, we see significant gains in min-recall (Tab. \ref{tab:min-recall}) just using the regularization term. We discuss their equivalance in Appendix Sec. \ref{Diagonal-G-hard-PL}.

\label{tab:post-shift}

\begin{table}[!t]
    \centering
    \small
    \caption{Results of maximizing the mean recall subject to coverage constraint 
    \emph{all classes} (col 2--3) and
    over the head and tail classes (col 4--7). 
   Proposed \texttt{CSST}(FixMatch) approach compares favorably to \ttt{ERM},\ttt{LA},\ttt{CSL} \ttt{vanilla}(FixMatch) and \texttt{CSST}(FixMatch) w/o KL-Thresh.. It is the best at both maximizing mean recall and coming close to satisfying the coverage constraint.
    }
    \begin{tabular}{lcc|cccccc}
        \hline
        \multicolumn{1}{c}{\textbf{Method}} & 
        \multicolumn{2}{c}{\textbf{CIFAR10-LT }} & 
        \multicolumn{2}{c}{\textbf{CIFAR100-LT }} &
        \multicolumn{2}{c}{\textbf{ImageNet100-LT }} &
        \\
        & \multicolumn{2}{c}{Per-class Coverage } & 
        \multicolumn{2}{c}{Head-Tail Coverage } &
        \multicolumn{2}{c}{Head-Tail Coverage } &
        \\
        & \multicolumn{2}{c}{($\rho = 100$, tgt : 0.1)} & 
        \multicolumn{2}{c}{($\rho = 10$, tgt : 0.01)} &
        \multicolumn{2}{c}{($\rho = 10$, tgt : 0.01)} &
        \\
        \hline
        & \textbf{Avg. Rec} & \textbf{Min. Cov} & 
        \textbf{Avg. Rec } & \textbf{Min. HT Cov} &
        \textbf{Avg. Rec } & \textbf{Min. HT Cov} &
        \\
        \hline
        \texttt{ERM} &  
            0.52 &	0.034 & 
            0.36 &	0.004 & 0.40 & 0.006
            \\
        \texttt{LA} &  
            0.51 &	0.039 & 
            0.36 &	0.009 & 0.48 & 0.009
            \\
        \texttt{CSL} &  
            0.60 &	0.090 & 
            0.45 &	0.010 & 0.48 & \highlight{0.010}
            \\
        \hline
        \begin{tabular}{@{}c@{}}\texttt{Vanilla} (FixMatch)
        \end{tabular}
            & 0.78 &	0.055 &
            \highlight{0.63} &	0.004 & \highlight{0.58} & 0.007
            \\
        \texttt{CSST}(FixMatch) w/o & & \highlight{} &  & & &  \\
        KL-Thresh.
            & \multirow{-2}{*}{0.67} &	\multirow{-2}{*}{\highlight{0.093}} &
            \multirow{-2}{*}{0.47} &	 \multirow{-2}{*}{0.010} &  \multirow{-2}{*}{0.26} &  \multirow{-2}{*}{0.010}
            \\
        \begin{tabular}{@{}c@{}}\texttt{CSST}(FixMatch) 
        \end{tabular}
            & \highlight{0.80} &	0.092 &
            \highlight{0.63} &	\highlight{0.010} & \highlight{0.58} & \highlight{0.010}
            \\
        \hline
    \end{tabular}
    \vspace{0.5em}

    \label{tab:coverage}
    \vspace{-5mm}
\end{table}

\begin{table}[!t]
    \centering
    \small
    \caption{Results of maximizing the min recall over
    \emph{all classes} for classification on NLP datasets. 
   Proposed \texttt{CSST}(UDA) approach outperforms ERM and \ttt{vanilla}(UDA) baselines.
    }
    \begin{tabular}{lcccccccc}
        \hline
        \multicolumn{1}{c}{\textbf{Method}} & 
        \multicolumn{2}{c}{\textbf{IMDb ($\rho = 10$)}} & 
        \multicolumn{2}{c}{\textbf{IMDb ($\rho = 100$)}} &
        \multicolumn{2}{c}{\textbf{DBpedia-14 ($\rho = 100$)}} &
        \\
        \hline
        & \textbf{Avg Rec} & \textbf{Min Rec} & 
        \textbf{Avg Rec } & \textbf{Min Rec} &
        \textbf{Avg Rec } & \textbf{Min Rec} &
         \\\hline
        \texttt{ERM} &  
            0.79 &	0.61 & 
            0.50 &	0.00 & 0.95 & 0.58
            \\
        \midrule
        \texttt{vanilla(UDA)}
            & 0.82 &	0.66 &
            0.50 &	0.00 & 0.96 & 0.65
            \\
        \texttt{CSST(UDA)}
            & \highlight{0.89} &	\highlight{0.88} &
            \highlight{0.77} &	\highlight{0.75} & \highlight{0.99} & \highlight{0.97}
            \\
        \hline
    \end{tabular}
    \vspace{-5mm}

    \label{tab:nlp_results}
    
\end{table}

\vspace{1mm} \noindent \textbf{Maximizing Mean Recall Under Coverage Constraints.}
\label{subsec:max-recall-under-cov}
Maximizing mean recall under coverage constraints objective seeks to result in a model with good average recall, yet at the same time constraints the proportion of predictions for each class to be uniform across classes.
The ideal target coverage under a balanced evaluation set(or such circumstances) is given as
$ 
\label{perfect-coverage}
\text{cov}_{i}[F] = \frac{1}{K}, \ \forall \  i \in [K]
$. 
Along similar lines to objective \eqref{eq:min-HT-recall-obj} we modify the objective \eqref{eq: cov-const-obj} to maximize the average recall subject to the both the average head and tail class coverage (HT Coverage) being above a given threshold of $\frac{0.95}{K}$. The $\bG$ is non-diagonal here and $G_{ij} = (\indc_{j \in \mathcal{H}}  \frac{\lambda_{\mathcal{H}}}{|\mathcal{H}|} + \indc_{j \in \mathcal{T}} \frac{\lambda_{\mathcal{T}}}{|\mathcal{T}|}) + \frac{\delta_{ij}}{K \pi_i} $.
\begin{equation}
 \label{eq: HT-cov-const-obj}
    \max_F \min_{(\lambda_{\mathcal{H}}, \lambda_{\mathcal{T}}) \in \RR^2_{\ge 0}} 
   \sum_{i\in [K]} \frac{C_{ii}[F]}{K\pi_i} + \lambda_{\mathcal{H}}\left(
      \sum_{i \in [K], j \in \mathcal{H}} \frac{C_{ij}[F]}{|\mathcal{H}|}  - \frac{0.95}{K}
   \right)  + \lambda_{\mathcal{T}}  \left(
      \sum_{i \in [K], j \in \mathcal{T}} \frac{C_{ij}[F]}{|\mathcal{T}|} - \frac{0.95}{K}
   \right).
\end{equation}
As these objectives corresponds to a non-diagonal $\bG$ as shown in Eq. \eqref{eq: cov-const-obj} in Sec. \ref{sec:ndo-loss-fun}. Hence, for introducing $\ttt{CSST}$ into FixMatch we replace first supervised loss $\mathcal{L}_{s}$ with $\mathcal{L}_{s}^{\mathrm{hyb}}$. For the unlabeled data we introduce $\ell^{\mathrm{hyb}}$ in $\mathcal{L}^{\mathrm{wt}}_{u}$ (Eq. \eqref{loss:CSST-KL}).  Hence, the final objective $\mathcal{L}$ is defined as, $\mathcal{L} =  \mathcal{L}_s^{\mathrm{hyb}} + \lambda_u \mathcal{L}^{\mathrm{wt}}_u $. We update the parameters of the cost-sensitive loss ($\bG$ and $\blambda$) periodically after few of SGD on the model parameters (Alg. \textcolor{red}{2} in Appendix). In this case our proposed thresholding mechanism in \ttt{CSST}(FixMatch)  introduced in Sec. \ref{subsec:threshold}, leads to effective utilization of unlabled data resulting in improved performance over the naive \ttt{CSST}(FixMatch) without (w/o) KL-Thresholding (Tab. \ref{tab:coverage}). In these experiments, the mean recall of our proposed approach either improves or stays same to the \ttt{vanilla}(FixMatch) implementation but only ours is the one that comes close to satisfying the coverage constraint. We further note that among all the methods, the only methods that come close to satisfying the coverage constraints are the ones with $\ell^{\mathrm{hyb}}$ included in them.

\vspace{1mm} \noindent \textbf{Ablation on amount of unlabeled data.} Here, we ablated the total number of labeled sampled while keeping the number of unlabeled samples constant.
We observe (in Fig.~\ref{fig:data-abl}) that as we increase the number of labeled samples the mean recall improved. This is because more labeled samples helps in better pseudo-labeling on the unlabeled samples and similarly as we decrease the number of labeled sample, the models errors on the pseudo-labels increase causing a reduction in mean recall. Hence, any addtional labeled data can be easily used to improve \ttt{CSST} performance.

\vspace{1mm} \noindent \textbf{Ablation on $\tau$ threshold.} Fig.~\ref{fig:thresh-abl} shows that when the KL divergence threshold is too high,  a large number of samples with very low degree of distribution match are used for generating the sharpened target (or pseudo-labels), this leads to worsening of mean recall as many targets are incorrect. We find that keeping a conservative of $\tau = 0.3$ works well across multiple experiments.

\begin{figure*}[t]
  \centering
\begin{minipage}[c]{0.34\linewidth}
        \centering
         \includegraphics[width=\textwidth]{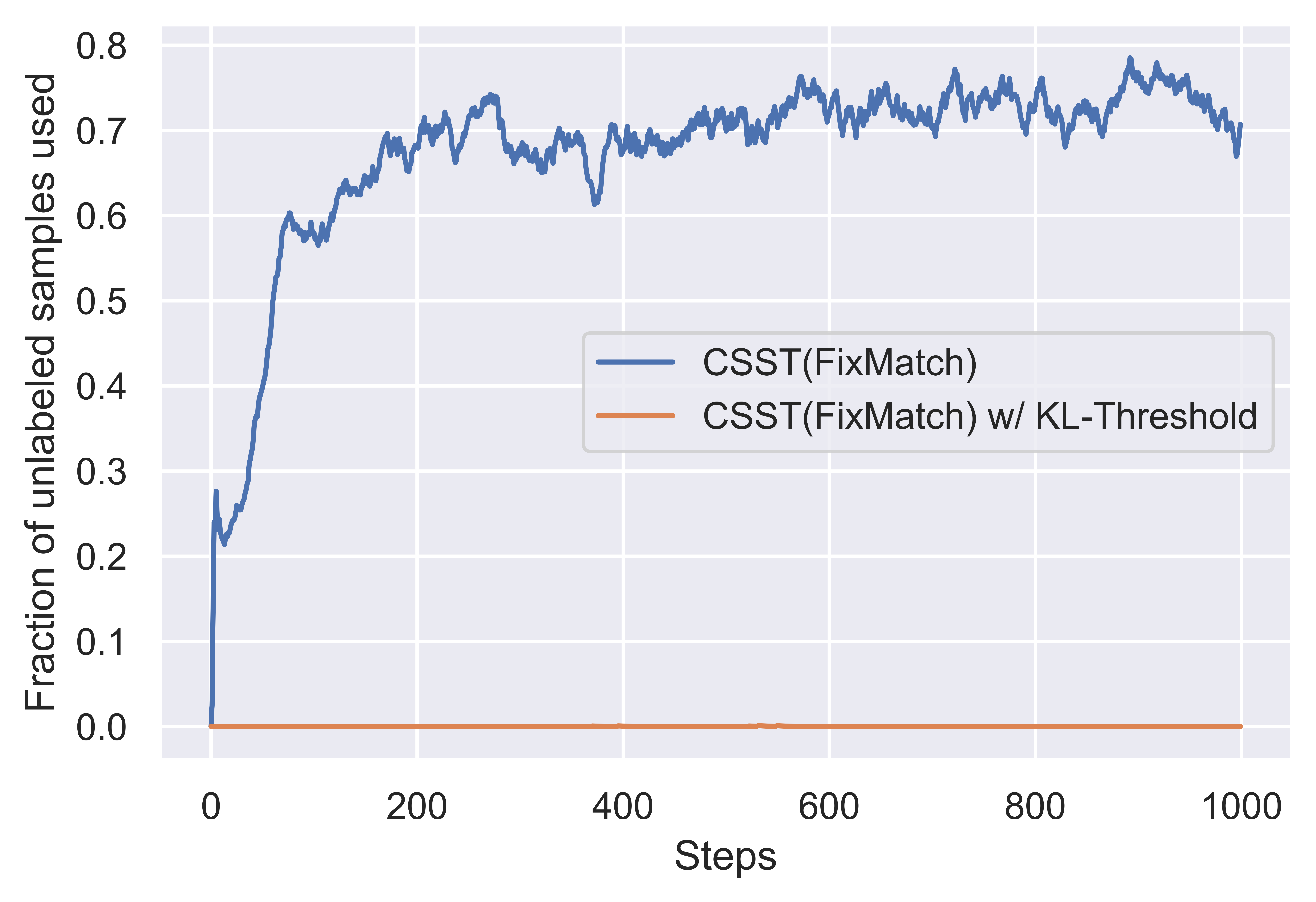}
         \vspace{-3mm}
         \caption{Fraction of unlabelled data used for maximizing average recall under coverage constraints for CIFAR-100 ($\rho=10,\mu=4$) (Sec. \ref{subsec:max-recall-under-cov}).}
         \label{fig:mask_rate}
  \end{minipage}
   \hfill
  \begin{minipage}[c]{0.64\textwidth}

\begin{subfigure}{.5\textwidth}
  \centering
  \includegraphics[width=\linewidth]{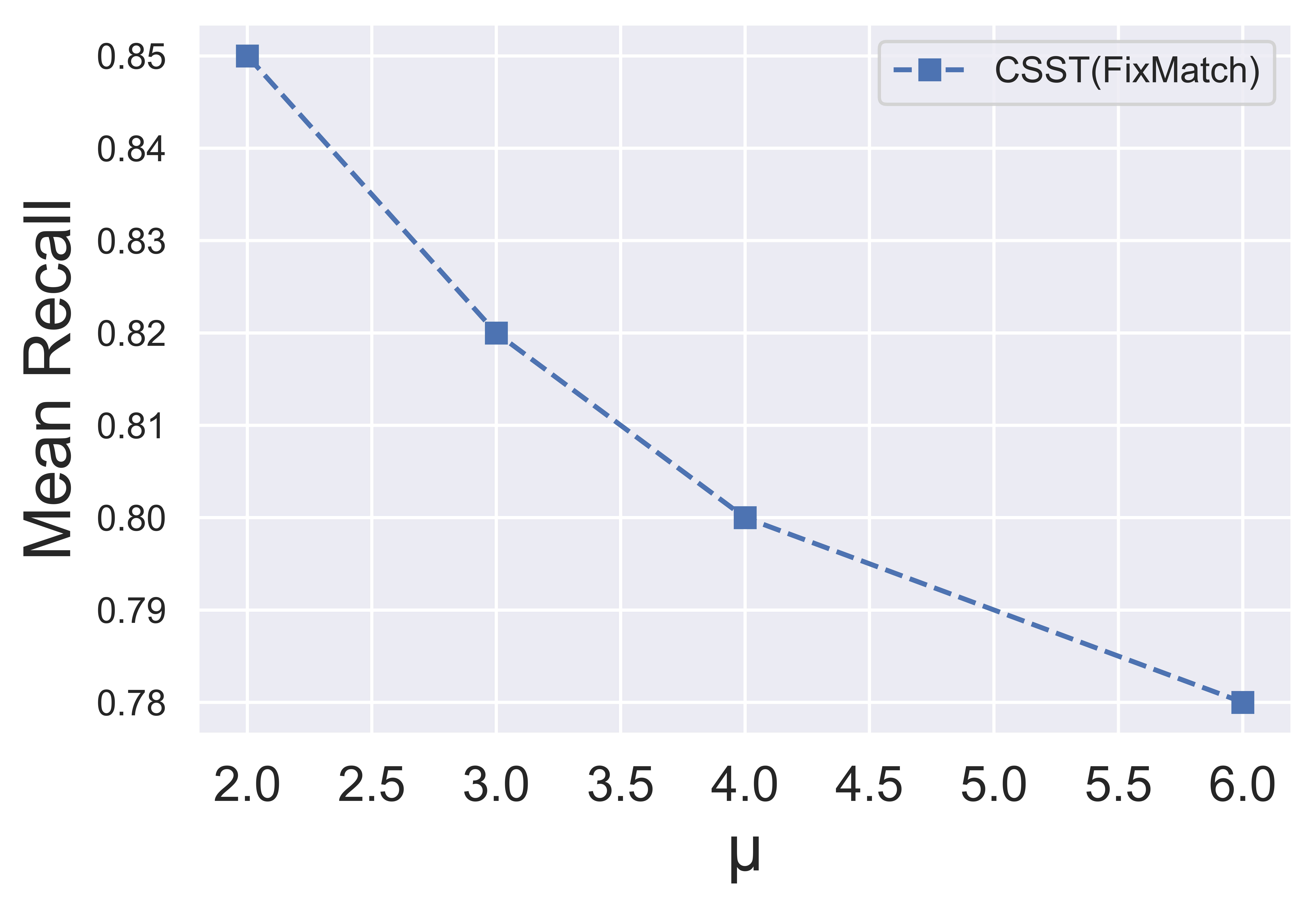}
  \caption{}
  \label{fig:data-abl}
\end{subfigure}%
\begin{subfigure}{.5\textwidth}
  \centering
  \includegraphics[width=\linewidth]{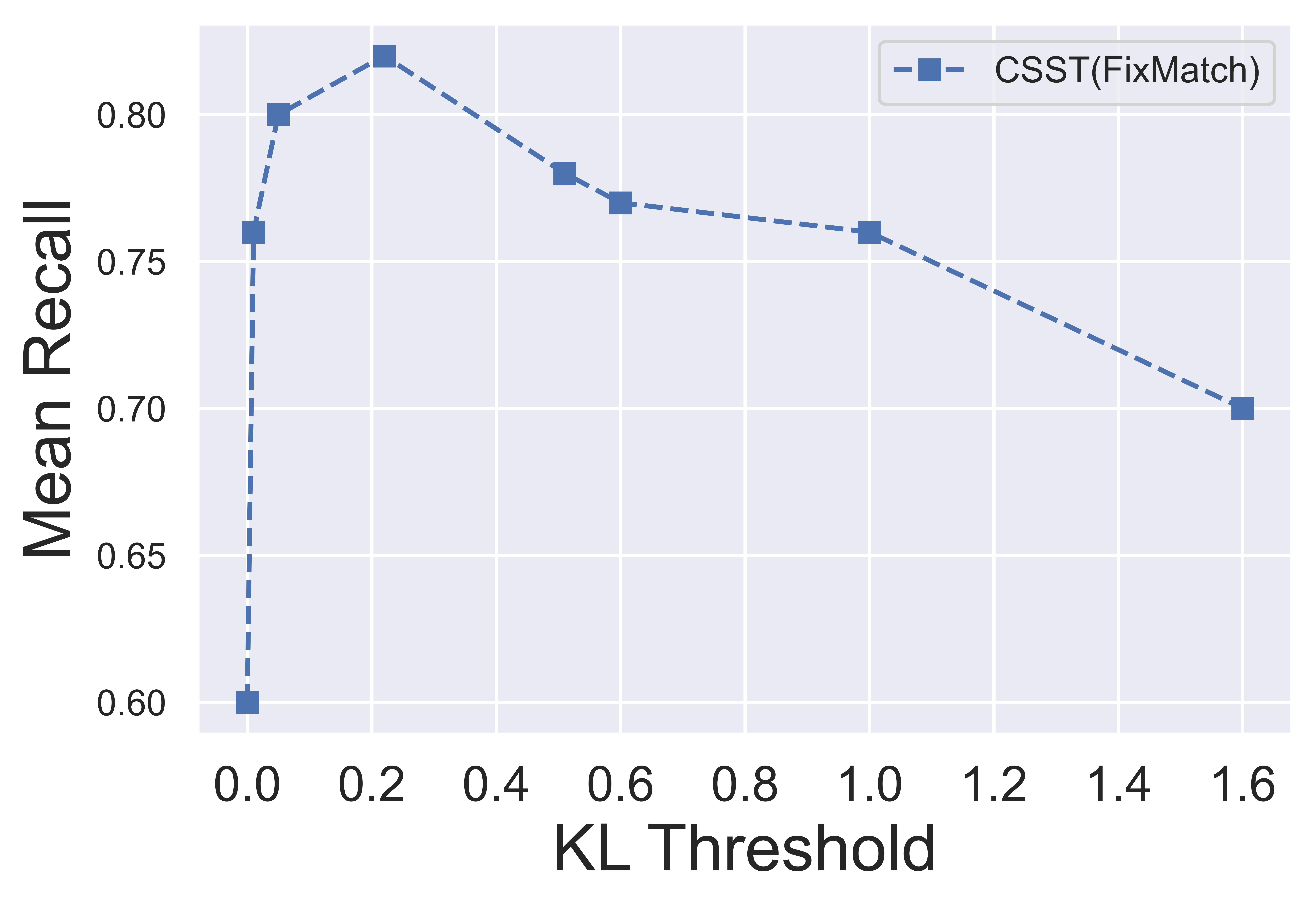}
  \caption{}
  \label{fig:thresh-abl}
\end{subfigure}
\vspace{-3mm}
\caption{Maximizing average recall under coverage constraints for CIFAR-10 Long tail ($\rho=100$) (Sec. \ref{subsec:max-recall-under-cov}). Fig.~shows comparison of (a) increasing the ratio of unlabeled samples to labeled samples given fixed number of unlabeled samples (b) Ablation on KL diveregence based threshold for \texttt{CSST}(FixMatch)}
\label{fig:ablation}
  \end{minipage}
\vspace{-4mm}
\end{figure*}

\section{Related Work}
\vspace{1mm}\noindent \textbf{Self-Training.}
Self-training algorithms have been popularly used for the tasks of semi-supervised learning~\cite{berthelot2019mixmatch,xie2020self,sohn2020fixmatch,laine2016temporal} and unsupervised domain adaptation~\cite{saito2017asymmetric,zou2019confidence}. In recent years several regularizers which enforce consistency in the neighborhood (either an adversarial perturbation~\cite{miyato2018virtual} or augmentation~\cite{xie2020unsupervised}) of a given sample have further enhanced the applicability and performance of self-training methods, when used in conjunction. However, these works have focused mostly on improving the generic metric of accuracy, unlike the general non-decomposable metrics we consider.

\noindent \textbf{Cost-Sensitive Learning.} It refers to problem settings where the cost of error differs for a sample based on what class it belongs to. These settings are very important for critical real world applications like disease diagnosis, wherein mistakenly classifying a diseased person as healthy can be disastrous. There have been a pletheora of techniques proposed for these which can be classified into: importance weighting~\cite{lin2002support,zadrozny2003cost,cui2019class} and adaptive margin~\cite{cao2019learning,zadrozny2003cost} based techniques. For overparameterized models \citet{narasimhan2015consistent} show that loss weighting based techniques are ineffective and propose a logit-adjustment based cost-sensitive loss which we also use in our framework.

\noindent \textbf{Complex Metrics for Deep Learning.} There has been a prolonged effort on optimizing more complex metrics that take into account practical constraints~\cite{narasimhan2014statistical,puthiya2014optimizing,natarajan2016optimal}. However most work has focused on linear models leaving scope for works in context of deep neural networks. \citet{sanyal2018optimizing} train DNN using reweighting strategies for optimizing metrics,  \citet{huang2019addressing} use a reinforcement learning strategy to optimize complex metrics, and \citet{kumar2021implicit}
optimize complex AUC (Area Under Curve) metric for a deep neural network. However, all these works have primarily worked in supervised learning setup and are not designed to effectively make use of available unlabeled data.

\section{Conclusion}
In this work, we aim to optimize practical non-decomposable metrics readily used in machine learning through self-training with consistency regularization, a class of semi-supervised learning methods. We introduce a cost-sensitive self-training framework (\texttt{CSST}) that involves minimizing a cost-sensitive error on pseudo labels and consistency regularization. We show theoretically that we can obtain classifiers that can better optimize the desired non-decomposable metric than the original model used for obtaining pseudo labels, under similar data distribution assumptions as used for theoretical analysis of Self-training. We then apply \texttt{CSST} to practical and effective self-training method of FixMatch and UDA, incorporating a novel regularizer and thresholding mechanism based on a given non-decomposable objective. We find that \texttt{CSST} leads to a significant gain in performance of desired non-decomposable metric, in comparison  to vanilla self-training-based baseline. Analyzing the \texttt{CSST} framework when the distribution of unlabeled data significantly differs from labeled data is a good direction to pursue for future work. 

\textbf{Acknowledgements:} This work was supported in part by Fujitsu Research grant. Harsh Rangwani is supported by Prime Minister's Research Fellowship (PMRF).

\newpage 

\appendix
\textbf{\Large Appendix}

\section{Limitations and Negative Societal Impacts}
\label{sec:limitations-and-negetive-societal}
\subsection{Limitations of our Work}
At this point we only consider optimizing objectives through \ttt{CSST} which can be written as a linear combination of entries of confusion matrix. Although there are important metrics like Recall, Coverage etc. which can be expressed as a linear form of confusion matrix. However, there do exist important metrics like Intersection over Union (IoU), Q-Mean etc. which we don't consider in the current work.
We leave this as an open direction for further work.

Also in this work we considered datasets where unlabeled data distribution doesn't significantly differ from the labeled data distribution, developing robust methods which can also take into account the distribution shift between unlabeled and labeled is an interesting direction for future work.

\subsection{Negative Societal Impact}
Our work has application in fairness domain~\cite{cotter2019optimization}, where it can be used to improve performance of minority sub-groups present in data. These fairness objectives can be practically enforced on neural networks through the proposed \ttt{CSST} framework.
However these same algorithms can be tweaked to artifically induce bias in decision making of trained neural networks, for example by ignoring performance of models on certain subgroups. Hence, we suggest deployment of these models after through testing on all sub groups of data.

\section{Connection between Minimization of Weighted Consistency Regularizer Loss (Eq. \eqref{eq:wt-consistency-reg}) and 
Theoretical Weighted Consistency ($R_{\mB, w}(F)$ in Sec. \ref{subsec:weighted_consistency})}
\label{sec:appendix-conn-between}

The theoretical weighted consistency regularization can be approximated under strong augmentation $\mA$ as $R_{\mB, w}(F) \approx \sum_{i,j\in[K]}w_{ij}\ex[x \sim P_i]{\bone(F(\mA(x)) \neq F(x))}$. Noting that $\bone(F(\mA(x)) \neq F(x)) \le \bone(F(\mA(x)) \neq j) + \bone(F(x) \neq j)$ for any class label $j$, 
this value of $R_{\mB, w}(F)$ is bounded as follows:
\begin{align*}
    R_{\mB, w}(F) &\approx \sum_{i,j\in[K]}w_{ij}\ex[x \sim P_i]{\bone(F(\mA(x)) \neq F(x))}\\
    & \le 
    \sum_{i, j \in [K]} w_{ij}\ex[x \sim P_i]{\bone(F(\mA(x)) \neq j))} + 
\sum_{i, j \in [K]} w_{ij}\ex[x \sim P_i]{\bone(F(x) \neq j)}.
\end{align*}

Assuming $F(x)$ to be composed of a classifier made of neural network model output $p_{m}(x)$. For consistency regularization, we choose samples which have a low second term by thresholding on confidence for which we obtain the pseudo-label $F(x) = \hat{p}_{m}(x)$. In that case the minimization of (the upper bound of) $R_{\mB, w}(F)$ is approximately equivalent to cost-sensitive learning (CSL):
\begin{equation*}
\min_{F}\sum_{i, j \in [K]} w_{ij}\ex[x]{\bone(F(\mA(x)) \neq j))}.
\end{equation*}
The above objective can be equivalently expressed in form of gain matrix $\bG$ where $G_{i,j} = w_{i,j}/(\pi_{i})$. Hence, the above minimization can be effectively done for different gain matrices based on corresponding non-decomposable metric~\ref{sec:ndo-loss-fun}. Following \citet{narasimhan2021training}, we can effectively minimize the $R_{\mB, w}(F)$ using the calibrated CSL loss function by plugging in the pseudo label $\hat{p}_m(x)$ in place of ground truth label ($\mathbf{y}$). The final weighted regularization function can be seen as:
\begin{equation}
    \ell^{hyb} (\hat{p}_{m}(x), p_{m}(\mA(x)); \bG)
\end{equation}
The above regularization loss is same as the weighted consistency regularization loss in Eq. \eqref{eq:wt-consistency-reg}, which leads to minimization of $\mathcal{D}_{KL}(\norm(\bG^{\mathbf{T}}\hat{p}_{m}(x)) || p_{m}(\mathcal{A}(x))) \; \forall x \in \mX $ as seen in Prop. \ref{prop:kl-weighted}. Hence, the minimization of theoretical consistency regularization term is achieved by minimizing the weighted consistency loss.

\section{Additional Examples and Proof of Theorem \ref{thm:main-err-bd}}
\label{sec:appendix-proofs}
In this section, we provide some examples for assumptions introduced in Sec. \ref{sec:theoretical_results}
and a proof of Theorem \ref{thm:main-err-bd}.
We provide proof of these examples in Sec. \ref{sec:misc-proofs}.

\subsection{Examples for Theoretical Assumptions}
\label{sec:appendix-examples-for-theoretical}
\addedtext{
The following example (Example \ref{exa:c-exp-mix-gauss}) shows that the $c$-expansion (Definition \ref{def:c-exp-def}) property is satisfied for mixtures of Gaussians and mixtures of manifolds.
\begin{example}
    \label{exa:c-exp-mix-gauss}
    By \cite[Examples 3.4, 3.5]{wei2020theoretical},
    the $c$-expansion property is satisfied for mixtures of 
    isotropic Gaussian distributions and mixtures of manifolds.
    More precisely, in the case of mixtures of isotropic Gaussian distributions, 
    i.e., if $Q$ is given as mixtures of $\mN(\tau_i, \frac{1}{d}I_{d \times d})$
    for $i=1, \dots, n$ with some $n \in \ZZ_{\ge 1}$
    and $\tau_i \in \RR^d$,
    and $\mB(x)$ is an $\ell_2$-ball with radius $r$
    then by \cite[(13)]{bobkov1997isoperimetric} and \cite[Section B.2]{wei2020theoretical},
    $Q$ satisfies the $c$-expansion property
    with 
    $c(p) = R_h(p)/p$ for $p > 0$
    and $h = 2r\sqrt{d}$
    (c.f., \cite[section B.2]{wei2020theoretical}).
    Here $R_h(p) = \Phi(\Phi^{-1}(p) + h)$
    and $\Phi$ is the cumulative distribution function of the 
    standard normal distribution on $\RR$. 
\end{example}
}

In Sec. \ref{subsec:weighted_consistency}, we required the assumption that $\gamma > 3$ (Assumption \ref{assump:pseudo-labeler})
and remarked that it roughly requires $\err_w(\gpl)$ is ``small''.
The following example provides explicit conditions for $\err_w(\gpl)$ that satisfy the assumption using a toy example.
\begin{example}
    \label{exa:gamma-3}
    Using a toy example provided in Example \ref{exa:c-exp-mix-gauss}, 
    we provide conditions that satisfy the assumption $\gamma > 3$ approximately.
    To explain the assumption $\gamma > 3$, we assume that 
    $\modifiedtext{\Pw}$ is given as a mixture of isotropic Gaussians and $\mB(x)$ is $\ell_2$-ball with radius $r$ 
    as in Example \ref{exa:c-exp-mix-gauss}.
    Furthermore, we assume that
    $|w|_1 = 1$ and $\err_w(\gstar)$ is sufficiently small compared to $\err_w(\gpl)$.
    Then, $p_w = \err_w(\gpl) + \err_w(\gstar) \approx \err_w(\gpl)$.
    Using this approximation,
    since $\modifiedtext{\Pw}$ satisfies the $c$-expansion property with $c(p) = R_{2r\sqrt{d}}(p)/p$,
    if $r = \frac{1}{2\sqrt{d}}$ then, the condition $\gamma > 3$ 
    is satisfied when $\err_w(\gpl) < 0.17$.
    If $r = \frac{3}{2\sqrt{d}}$ then, the condition $\gamma > 3$
     is satisfied when $\err_w(\gpl) < 0.33$.
\end{example}

In Assumption \ref{assump:mu-small}, we assumed that both of $\err_w(\gstar)$ and $R_{\mB, w}(\gstar)$ are small.
The following example suggests the validity of this assumption. 
\begin{example}
    \label{exa:r_gopt_bound}
    In this example, 
    we assume $w$ is a diagonal matrix $\diag(w_1, \dots, w_K)$.
    For simplicity, we normalize $w$ so that $\sum_{i\in [K]}w_i = 1$.
    As in \cite[Example 3.4]{wei2020theoretical}, we assume 
    that $P_i$ is given as isotropic Gaussian distribution $\mN(\tau_i, \frac{1}{d}I_{d\times d})$
    with $\tau_i \in \RR^d$ for $i = 1,\dots, K$ and $\mB(x)$ is an $\ell^2$-ball with radius $\frac{1}{2\sqrt{d}}$.
    Furthermore, 
    we assume $\inf_{1 \le i < j \le K}\|\tau_i - \tau_j\|_2 \gtrsim \frac{\sqrt{\log d}}{\sqrt{d}}$
    and $\sup_{i, j \in [K]}\frac{w_i}{w_j} = o(d)$, where the latter assumption is valid for high dimensional 
    datasets (e.g., image datasets).
    Then it can be proved that there exists a classifier $F$
    such that $R_{\mB, w}(F) = O(\frac{1}{d^c})$ and $ \err_w(F) = O(\frac{1}{d^c})$,
    where $c > 0$ is a constant (we can take $F$ as the Bayes-optimal classifier for $\err_w$).
    Thus, this suggests that Assumption \ref{assump:mu-small} is valid for 
    datasets with high dimensional instances.
\end{example}

\addedtext{The statement of Example \ref{exa:gamma-3} follows from numerical computation of $R_{2r\sqrt{d}}(p)/p$. We provide proofs of Examples \ref{exa:c-exp-mix-gauss} and \ref{exa:r_gopt_bound} in Sec. \ref{sec:misc-proofs}.}

\subsection{Proof of Theorem \ref{thm:main-err-bd} Assuming a Lemma}
Theorem \ref{thm:main-err-bd} can be deduced from the following lemma (by taking $H = \gstar$ and 
$\mL_{Q, H}(\widehat{F}) \le \mL_{Q, H}(\gstar)$),
which provides a similar result to \cite[Lemma A.8]{wei2020theoretical}.
\begin{lemma}
    \label{lem:err-bd}
    Let $H$ be a classifier and 
    $Q$ a probability measure on $\mX$ satisfying $c$-expansion property.
    We put $\gamma_H = c(Q(\{x \in \mX: \gpl(x) \ne H(x)\}))$.
    For a classifier $F$, we define $\mS_{\mB}(F)$ by 
    \begin{math}
    \mS_{\mB}(F) = \{x \in \mX: F(x) = F(x') \quad \forall x' \in \mB(x)\}.
    \end{math}
    For a classifier $F$, we define $\mL_{Q, H}(F)$ by
    \begin{multline*}
        \frac{\gamma_H + 1}{\gamma_H - 1}
        Q(\{x \in \mX : F(x) \ne \gpl(x)\})\\
        + \frac{2\gamma_H}{\gamma_H - 1}
        Q(\mS_{\mB}^c(F)) + 
        \frac{2\gamma_H}{\gamma_H - 1}Q(\mS_\mB^c(H))
         - Q(\{x \in \mX: \gpl(x) \ne H(x)\}),
    \end{multline*}
    where $\mS_\mB^c(F)$ denotes the complement of $\mS_\mB(F)$.
    Then, we have $Q(\{x \in \mX: F(x) \ne H(x)\}) \le \mL_{Q, H}(F)$
    for any classifier $F$.
\end{lemma}
In this subsection,  we provide a proof of Theorem \ref{thm:main-err-bd}
assuming Lemma \ref{lem:err-bd}.
We provide a proof of the lemma in the next subsection.
For a classifier $F$, we define $\mM(F)$ as $\{x \in \mX: F(x) \ne \gstar(x)\}$
and $\mpl(F)$ as $\{x \in \mX: F(x) \neq \gpl(x)\}$.
We define $\widetilde{\mL}_w(F)$ by 
\begin{align*}
    \widetilde{\mL}_w(F)
    &= \mL_w(F) +  
    \frac{2\gamma}{\gamma - 1} R_{\mB, w}(\gstar) - \modifiedtext{\Pw}(\{x \in \mX: \gpl(x) \ne \gstar(x)\}).
\end{align*}
We note that $\widetilde{\mL}_w(F) - \mL_w(F)$ does not depend on $F$.

\begin{proof}[Proof of Theorem \ref{thm:main-err-bd}]
    We let $Q = \modifiedtext{\Pw}$ and $H = \gstar$ in Lemma \ref{lem:err-bd}
    and denote $\gamma_H$ in the lemma by $\gamma'$.
    Since $w_{ij}\ge 0$ and $\ex[x \sim P_i]{\indc(\gpl(x) \ne \gstar(x))} 
    \le \ex[x \sim P_i]{\indc(\gpl(x) \ne j)} + \ex[x \sim P_i]{\indc(\gstar(x) \ne j)}$
     for any $i, j$, we have the following: 
    \begin{align*}
        |w|_1 \modifiedtext{\Pw}(\mM(\gpl)) &= \sum_{i, j\in[K]}w_{ij}\ex[x \sim P_i]{\indc(\gpl(x) \ne \gstar(x))}\\
        &\le \sum_{i, j \in [K]}w_{ij}
        \left\{\ex[x \sim P_i]{\indc(\gpl(x) \ne j)} + \ex[x \sim P_i]{\indc(\gstar(x) \ne j)}\right\}\\
        &= \err_w(\gpl) + \err_w(\gstar).
    \end{align*}
    Here, the first equation follows from the definition of $\Pw$.
    Thus, we obtain $\modifiedtext{\Pw}(\mM(\gpl)) \le p_w$.
    By definition of $\gamma$ and $\gamma'$
    ($\gamma = c(p_w)$ and $\gamma' = c\left(\Pw(\mM(\gpl))\right)$)
    and the assumption that $c$ is non-increasing, we have the following: \begin{equation}
        \gamma \le \gamma'.
        \label{eq:gamma-gamma'}
    \end{equation}
    We note that 
    \begin{align}
        \err_w(F) 
        &= \sum_{i, j\in [K]}
        w_{ij}\ex[x\sim P_i]{\indc(F(x) \ne j)}\notag \\
        &\le 
        \sum_{i, j \in [K]} 
        w_{ij} \ex[x\sim P_i]{\indc(F(x) \ne \gstar(x))}
        + 
        \sum_{i, j \in [K]} 
        w_{ij} \ex[x\sim P_i]{\indc(\gstar(x) \ne j)}\notag \\
        &= |w|_1 \modifiedtext{\Pw}(\mM(F)) + \err_w(\gstar).
        \label{eq:errw_pw_errw_gstar}
    \end{align}
    Here the first inequality holds since 
    $\indc(F(x) \neq j) \le \indc(F(x) \neq \gstar(x)) + \indc(\gstar(x) \neq j)$
    for any $x$ and $j$.
    By \eqref{eq:errw_pw_errw_gstar} and Lemma \ref{lem:err-bd}, the error is upper bounded as follows:
    \begin{multline*}
        \err_w(F)
        \le 
        \err_w(\gstar) + 
        \frac{\gamma' + 1}{\gamma' - 1} |w|_1 \modifiedtext{\Pw}(\mpl(F)) \\
        +\frac{2\gamma'}{\gamma' - 1} |w|_1 \modifiedtext{\Pw} (\mS_\mB^c(F)) +
        \frac{2\gamma'}{\gamma' - 1}|w|_1 \modifiedtext{\Pw}(\mS_\mB^c(\gstar)) - |w|_1 \modifiedtext{\Pw}(\mM(\gpl)).
    \end{multline*}
    By \eqref{eq:gamma-gamma'}, we obtain
    \begin{multline}
        \label{eq:errw-lem-applied}
        \err_w(F)
        \le 
        \err_w(\gstar) + 
        \frac{\gamma + 1}{\gamma - 1} |w|_1 \modifiedtext{\Pw}(\mpl(F)) \\
        +\frac{2\gamma}{\gamma - 1} |w|_1 \modifiedtext{\Pw} (\mS_\mB^c(F)) +
        \frac{2\gamma}{\gamma - 1}|w|_1 \modifiedtext{\Pw}(\mS_\mB^c(\gstar)) - |w|_1 \modifiedtext{\Pw}(\mM(\gpl)).
    \end{multline}
    By definition of $\mL_w$ and letting $F = \widehat{F}$, we have the following:
    \begin{align*}
        & \err_w(\widehat{F})
        \le \err_w(\gstar)  + \mL_w(\widehat{F}) + \frac{2\gamma}{\gamma - 1} R_{\mB, w}(\gstar) - |w|_1 \modifiedtext{\Pw}(\{x \in \mX: \gpl(x) \ne \gstar(x)\})\\
        &\le \err_w(\gstar) + \mL_w(\gstar) + \frac{2\gamma}{\gamma - 1} R_{\mB, w}(\gstar) - |w|_1 \modifiedtext{\Pw}(\{x \in \mX: \gpl(x) \ne \gstar(x)\})\\
        &= \err_w(\gstar) + \frac{2}{\gamma - 1}|w|_1 \modifiedtext{\Pw}(\mM(\gpl)) + \frac{4\gamma}{\gamma - 1} R_{\mB, w}(\gstar)\\
        &\le \err_w(\gstar) + \frac{2}{\gamma - 1}(\err_w(\gpl) + \err_w(\gstar)) + \frac{4\gamma}{\gamma - 1}R_{\mB, w}(\gstar)\\
        & = 
        \frac{2}{\gamma - 1}
        \err_w(\gpl)
        + \frac{\gamma + 1}{\gamma - 1} \err_w(\gstar)
        + \frac{4\gamma}{\gamma - 1} R_{\mB, w}(\gstar).
    \end{align*}
    Here, the second inequality holds since $\widehat{F}$ is a minimizer of $\mL_w$,
    the third inequality follows from 
    $\indc(\gstar(x) \ne \gpl(x)) \le \indc(\gstar(x) \ne j) + \indc(\gpl(x) \ne j)$ for any $j$.
    Thus, we have the assertion of the theorem.
\end{proof}

\subsection{Proof of Lemma \ref{lem:err-bd}}
We decompose $\mM(F) \cap \mS_\mB(F) \cap \mS_\mB(H)$ into the following three sets:
\begin{align*}
    \mN_1 &= \{x \in \mS_\mB(F) \cap \mS_\mB(H): F(x) = \gpl(x), \text{ and } \gpl(x) \ne H(x)\},\\
    \mN_2 &= \{x \in \mS_\mB(F) \cap \mS_\mB(H): F(x) \ne \gpl(x), \gpl(x) \ne H(x), \text{ and } F(x) \ne H(x)\},\\
    \mN_3 &= \{x \in \mS_\mB(F) \cap \mS_\mB(H): F(x) \ne \gpl(x) \text{ and } \gpl(x) = H(x)\}.
\end{align*}

\begin{lemma}
    \label{lem:incl}
    Let $S = \mS_\mB(F) \cap \mS_\mB(H)$
    and $V = \mM(F) \cap \mM(\gpl)\cap S$.
    Then, we have $\mN(V) \cap \mM^c(F) \cap S = \emptyset$
    and $\mN(V) \cap \mM^c(\gpl) \cap S \subseteq \mpl(F)$.
    Here $\mpl(F)$ is defined as $\{x \in \mX : F(x) \neq \gpl(x)\}$.
\end{lemma}
\begin{proof}
    We take any element $x$ in $\mN(V) \cap S$.
    Since $x \in \mN\left(V\right)$ and definition of neighborhoods,
    there exists $x' \in \mM(\gpl) \cap \mM(F) \cap S$ such that 
    $\mB(x) \cap \mB(x') \neq \emptyset$.
    Since $x, x' \in \mS_\mB(F)$, 
    $F$ takes the same values on $\mB(x)$ and $\mB(x')$.
    By $\mB(x) \cap \mB(x') \neq \emptyset$, $F$ takes the same value on $\mB(x) \cup \mB(x')$.
    It follows that $F(x) = F(x')$.
    Since we have $x, x' \in \mS_\mB(H)$, similarly, we see that $H(x) = H(x')$.
    By $x' \in \mM(F)$, we have $F(x) = F(x') \neq H(x') = H(x)$.
    Therefore, we have $x \in \mM(F)$.
    Thus, we have proved that 
    $x \in \mN(V) \cap S$ implies $x \in \mM(F)$, i.e.,
    $ \mN(V) \cap S \subseteq \mM(F)$.
    Therefore, we have $\mN(V) \cap \mM^c(F) \cap S = \emptyset$.
    This completes the proof of the first statement.
    Next, we assume 
    $x \in \mN(V) \cap \mM^c(\gpl) \cap S$.
    Then, we have $F(x)\neq H(x)$ and $\gpl(x) = H(x)$.
    Therefore, we obtain $F(x) \neq \gpl(x)$.
    This completes the proof.
\end{proof}

\begin{lemma}
    \label{lem:n1_cup_n2}
    Suppose that assumptions of Lemma \ref{lem:err-bd} hold.
    We define $q$ as follows:
    \begin{equation}
        \label{eq:def-of-q}
        q = \frac{Q(\mpl(F) \cup \mS_\mB^c(F) \cup \mS_\mB^c(H))}{\gamma_H - 1}.
    \end{equation}
    Then, we have $Q(\mS_\mB(F) \cap \mS_\mB(H) \cap \mM(\gpl) \cap \mM(F)) \le q$.
    In particular, noting that 
    $\mN_1 \cup \mN_2 \subseteq \mS_\mB(F) \cap \mS_\mB(H) \cap \mM(\gpl)\cap \mM(F)$, 
    we have $Q(\mN_1 \cup \mN_2) \le q$.
\end{lemma}
\begin{proof}
    We let $S = \mS_\mB(F) \cap \mS_\mB(H)$ and $V = \mM(F) \cap \mM(\gpl)\cap S$ as before.
    Then by Lemma \ref{lem:incl}, we have
    \begin{align*}
        \mN(V) \cap V^c \cap S &= 
        \left(\mN(V) \cap \mM^c(F) \cap S\right)\cup
        \left(\mN(V) \cap \mM^c(\gpl) \cap S\right)\\
        &\subseteq
        \emptyset \cup \mpl(F) = \mpl(F).
    \end{align*}
    Therefore, we have
    \begin{align*}
        \mN(V) \cap V^c &= \mN(V) \cap V^c \cap (S \cup S^c)\\
        &= \left(\mN(V)  \cap V^c \cap S \right)
        \cup \left(
            \mN(V)  \cap V^c \cap S^c
        \right)\\
        &\subseteq 
        \mpl(F) \cup S^c.
    \end{align*}
    Thus, by the $c$-expansion property, we have
    \begin{align*}
        Q(\mpl(F) \cup S^c) &\ge Q(\mN(V) \cap V^c)\\
        & \ge Q(\mN(V)) - Q(V)\\
        & \ge \left(c(Q(V)) - 1\right)Q(V).
    \end{align*}
    Since $V \subseteq \mM(\gpl)$, $c$ is non-increasing, and $\gamma_H > 1$,
    we have $Q(V) \le Q(\mpl(F) \cup S^c)/(\gamma_H - 1) \le q$.
    This completes the proof.
\end{proof}

The following lemma provides an upper bound of $Q(\mN_3)$.
\begin{lemma}
    \label{lem:m3-upper-bd}
    Suppose that the assumptions of Lemma \ref{lem:err-bd} hold.
    We have 
    \begin{equation*}
        Q(\mN_3) \le q + Q(\mS_\mB^c(F) \cup \mS_\mB^c(H)) + Q(\mpl(F)) - Q(\mM(\gpl)),
    \end{equation*}
    where $q$ is defined by \eqref{eq:def-of-q}.
\end{lemma}
\begin{proof}
    We let $S = \mS_\mB(F) \cap \mS_\mB(H)$.
   First, we prove 
   \begin{equation}
    \label{eq:m3_m1_disjoint_union}
       \mN_3 \sqcup \left(\mpl^c(F) \cap S \right) 
       = \mN_1 \sqcup \left(\mM^c(\gpl) \cap S \right).
   \end{equation} 
   Here, for sets $A, B$, we denote union $A \cup B$ by $A \sqcup B$ if the union is disjoint. 
   By definition, we have $\mN_1 = S \cap \mpl^c(F) \cap \mM(\gpl)$
   and $\mN_3 = S \cap \mpl(F) \cap \mM^c(\gpl)$.
   Thus, we have
   \begin{align*}
       &\mN_3 \cup \left(\mpl^c(F) \cap S \right)\\
       &= \left(
        S \cap \mpl(F) \cap \mM^c(\gpl)
       \right)  \cup \left(\mpl^c(F) \cap S \right)\\
       &= S \cap 
       \left\{
       \left(\mpl(F)\cap \mM^c(\gpl) \right)
       \cup \mpl^c(F)
       \right\}\\
       &= S\cap \left(\mM^c(\gpl) \cup \mpl^c(F) \right).
   \end{align*}
   Similarly, we have the following:
   \begin{align*}
       &\mN_1 \cup     \left(\mM^c(\gpl) \cap S\right)\\
       &= \left(
        S \cap \mpl^c(F) \cap \mM(\gpl)
       \right)
       \cup
     \left(\mM^c(\gpl) \cap S \right)\\
       &= S \cap \left(\mM^c(\gpl) \cup \mpl^c(F) \right).
   \end{align*}
   Since disjointness is obvious by definition, we obtain \eqref{eq:m3_m1_disjoint_union}.
   Next, we note that the following holds:
   \begin{align}
      Q(\mpl^c(F) \cap S) &= Q(\mpl^c(F)) - Q(\mpl^c(F) \cap S^c) \notag \\
      & \ge Q(\mpl^c(F)) - Q(S^c).
      \label{eq:lem-m3-upper-bd1}
   \end{align}
   By \eqref{eq:m3_m1_disjoint_union}, we obtain the following:
   \begin{align*}
       Q(\mN_3) &= Q(\mN_1) + Q(\mM^c(\gpl) \cap S)
       - Q(\mpl^c(F) \cap S)\\
       &\le Q(\mN_1) + Q(\mM^c(\gpl)) - Q(\mM^c(\gpl) \cap S)\\
       & \le Q(\mN_1) + Q(\mM^c(\gpl)) - Q(\mpl^c(F)) + Q(S^c)\\
       & \le q + Q(\mM^c(\gpl)) - Q(\mpl^c(F)) +Q(S^c)\\
       & = q - Q(\mM(\gpl)) + Q(\mpl(F))  + Q(S^c).
   \end{align*} 
   Here the second inequality follows from \eqref{eq:lem-m3-upper-bd1}
   and the third inequality follows from Lemma \ref{lem:n1_cup_n2}.
   This completes the proof.
\end{proof}

Now, we can prove Lemma \ref{lem:err-bd} as follows.
\begin{proof}[Proof of Lemma \ref{lem:err-bd}]
   \begin{align*}
       Q(\mM(F)) 
       &= Q(\mM(F) \cap \mS_\mB(F) \cap \mS_\mB(H)) + 
       Q\left(\mM(F) \cap \left(\mS_\mB^c(F) \cup \mS_\mB^c(H) \right)\right)\\
       &\le 
       Q(\mN_1 \cup \mN_2) + Q(\mN_3) + Q(\mS_\mB^c(F) \cup \mS_\mB^c(H))\\
       &\le 2q + 2Q(\mS_\mB^c(F) \cup \mS_\mB^c(H)) + Q(\mpl(F)) - Q(\mM(\gpl)).
   \end{align*} 
   Here, the last inequality follows from Lemmas \ref{lem:n1_cup_n2} and \ref{lem:m3-upper-bd}.
   Since $q$ satisfies $q \le \frac{Q(\mpl(F)) + Q(\mS_\mB^c(F) + Q(\mS_\mB^c(H)))}{\gamma_H - 1}$
   by \eqref{eq:def-of-q}, we have our assertion.
\end{proof}
\subsection{Miscellaneous Proofs for Examples}
\label{sec:misc-proofs}
\begin{proof}[Proof of Example \ref{exa:c-exp-mix-gauss}]
    In Example \ref{exa:c-exp-mix-gauss}, we stated that 
    $p \mapsto R_h(p)/p$ is non-increasing.
    This follows from the concavity of $R_h$
    and $\lim_{p \rightarrow +0} R_h(p) = 0$.
    In fact, we can prove the concavity of $R_h$
    by $\frac{d^2R_h}{dp^2}(p) = - h \exp\left(\frac{\xi^2}{2} - h\xi - \frac{1}{2}h^2\right) \le 0$,
    where $\xi = \Phi^{-1}(p)$.
\end{proof}

\begin{proof}[Proof of Example \ref{exa:r_gopt_bound}]
    For each $i, j \in [K]$, $w_i P_i(x) \ge w_j P_j(x)$ is equivalent to
    $(x - \tau_i) \cdot v_{ji} \le \frac{\|\tau_i - \tau_j\|}{2} +
     \frac{2 (\log w_i - \log w_j)}{d \|\tau_i - \tau_j\|}$,
     where $v_{ji} = \frac{\tau_j - \tau_i}{\|\tau_i - \tau_j\|}$.
    Thus, for each $i \in [K]$,
    we have $\bigcap_{j \in [K]\setminus \{i\}}X_{ij} \subseteq \mS_\mB(\gopt)$.
    Here $X_{ij}$ is defined as $\{x \in \mX : (x - \tau_i)\cdot v_{ji}  \le 
    \frac{\|\tau_i - \tau_j\|}{2} + 
    \frac{2 (\log w_i - \log w_j)}{d \|\tau_i - \tau_j\|} - \frac{r}{2}
    \}$.
    For any $w \in \RR^d$ with $\|w\|_2 = \sqrt{d}$ and $a > 0$,
    we have $P_i(\{x \in \mX: (x - \tau_i)\cdot w > a\}) = 1- \Phi(a) \le \frac{1}{2}\exp(-a^2/2)$
    (c.f., \cite{chiani2003new}).
    Thus, $P_i(X_{ij}^c) \le \frac{1}{2}\exp(-d a_{ij}^2/2)$, where 
    $a_{ij} = \frac{\|\tau_i - \tau_j\|}{2} + 
    \frac{2 (\log w_i - \log w_j)}{d \|\tau_i - \tau_j\|} - \frac{r}{2}$.
    By assumptions, we have $a_{ij}\sqrt{d} \gtrsim \sqrt{\log d}$.
    Therefore, $P_i(X_{ij}^c) = O(\frac{1}{poly(d)})$.
    It follows that
    \begin{math}
        P_i(\mS_\mB^c(\gopt)) \le \sum_{j\in [K]\setminus \{i\}} P_i(X_{ij}^c)
        = O(\frac{1}{poly(d)})
    \end{math}.
    Thus, we have $R_{\mB, w}(\gopt) = O(\frac{1}{poly(d)})$.
    By the same way, we can prove that $\err_w(\gopt) = O(\frac{1}{poly(d)})$.
\end{proof}

\section{All-Layer Margin Generalization Bounds}
\label{sec:appendix-all-layer}
Following \cite{wei2020theoretical,wei2019improved}, we introduce 
all layer margin of neural networks and provide generalization bounds of \csst{}. In this section, we assume that classifier $F(x)$ is given as $F(x) = \argmax_{1\le i \le K} \Phi_i(x)$,
where $\Phi$ is a neural network of the form
\begin{equation*}
\Phi(x) = (f_p \circ f_{p-1} \circ \cdots \circ f_1) (x).
\end{equation*}
Here $f_i : \RR^{d_{i-1}} \rightarrow \RR^{d_{i}}$ with $d_0 = d$ and $d_p = K$.
We assume that each $f_i$ belongs to a function class $\mF_i \subset \mathrm{Map}(\RR^{d_{i-1}}, \RR^{d_{i}})$.
We define a function class $\mF$ to which $\Phi$ belongs 
by 
\begin{equation*}
  \mF = \{\Phi: \RR^{d} \rightarrow \RR^K : \Phi(x) = (f_p \circ \dots \circ f_1)(x), \quad f_i \in \mF_i,  
  \forall i \}.
\end{equation*}
For example, for $b > 0$, 
$\mF_i$ is given as $\{h \mapsto W \phi(h) : W \in \RR^{d_{i-1} \times d_{i}}, \|W\|_\fro \le b\}$ if $i > 1$
and $\{h \mapsto W h : W \in \RR^{d_0 \times d_{1}}, \|W\|_\fro \le b\}$ if $i = 1$,
where $\phi$ is a link function (applied on $\RR^{d_i}$ entry-wise) with bounded operator norm
(i.e., $\|\phi\|_\op := \sup_{x \in \RR^{d_i} \setminus \{0\}}\|\phi(x)\|_2/\|x \|_2 < \infty$) 
and $\|W\|_\fro$ denotes the Frobenius norm of the matrix.
However, we do not assume the function class $\mF_i$ does not have this specific form.
We assume that each function class $\mF_i$ is a normed vector space with norm $\| \cdot \|$.
In the example above, we consider the operator norm, i.e., if $f(h) = \phi(Wh)$, $\|f\|$ is defined 
as $\| f\|_\op$.
Let $x_1, \dots, x_n$ be a finite i.i.d. sequence of samples drawn from $\modifiedtext{\Pw}$.
We denote the corresponding empirical distribution by $\pwhat$, i.e.,
for a measurable function $f$ on $\mX$, $\ex[x \sim \pwhat]{f} = \sum_{i=1}^{n}f(x_i)$.

For $\xi = (\xi_1, \dots, \xi_p) \in \prod_{i=1}^p\RR^{d_i}$, we define the perturbed output 
$\Phi(x, \xi)$ as $\Phi(x, \xi) = h_p(x, \xi)$, where
\begin{align*}
    h_1(x, \xi) &= f_1(x) + \xi_1 \|x \|_2,\\ 
    h_i(x, \xi) &= f_i(h_{i-1}(x, \xi)) + \xi_{i} \|h_{i-1}(x, \xi)\|_2, \quad \text{for } 2 \le i \le p.
\end{align*}
Let $x \in \mX$ and $y \in [K]$. 
We define $\Xi(\Phi, x, y)$ by $\{\xi \in \prod_{i=1}^p \RR^{d_i}: \argmax_{i}\Phi_i(x, \xi) \neq y\}$.
Then, the all-layer margin $m(\Phi, x, y)$ is defined as 
\begin{equation*}
    m(\Phi, x, y) = \min_{\xi \in \Xi(\Phi, x, y)}  \|\xi\|_2,
\end{equation*}
where $\|\xi\|_2$ is given by $\sqrt{\sum_{i=1}^p\|\xi_i\|^2_2}$.
Following \cite{wei2020theoretical}, we define 
a variant of the all-layer margin that measures 
robustness of $\Phi$ with respect to input transformations defined by $\mB(x)$ as follows:
\begin{equation*}
    m_{\mB}(\Phi, x) := \min_{x' \in \mB(x)} m(F, x', \argmax_{i} \Phi_i(x)).
\end{equation*}
\begin{assumption}[c.f. \cite{wei2019improved}, Condition A.1]
    \label{assump:cov-cond}
    Let $\mG$ be a normed space with norm $\| \cdot \|$
    and $\epsilon > 0$.
    We say $\mG$  satisfies the $\epsilon^{-2}$ covering condition with complexity $\mcnorm(\mG)$ if 
    for all $\epsilon > 0$, we have
    \begin{equation*}
        \log \mnnorm(\epsilon, \mG)  \le \frac{\mcnorm(\mG)}{\epsilon^2}.
    \end{equation*}
    Here $\mnnorm(\epsilon, \mG)$ the $\epsilon$-covering number of $\mG$.
    We assume function class $\mF_i$ satisfies the $\epsilon^{-2}$ covering condition with complexity $\mcnorm(\mF_i)$
    for each $1 \le i \le p$.
\end{assumption}
Throughout this section, we suppose that Assumption \ref{assump:cov-cond} holds.
Essentially, the following two propositions follows were proved by \citet{wei2020theoretical}:
\begin{proposition}[c.f., \cite{wei2020theoretical}, Lemma D.6]
    \label{prop:rw-finite-bd}
   With probability at least $1-\delta$ over the draw of the training data, for all $t \in (0, \infty)$,
   any $\Phi \in \mF$ satisfies the following:
   \begin{equation*}
       R_{\mB, w}(F) = \ex[\pwhat]{\bone(m_{\mB}(\Phi, x) \le t)} +
       \otilde\left(\frac{\sum_{i=1}^p \mcop(\mF_i)}{t \sqrt{n}}\right) +
       \zeta,
   \end{equation*} 
   where $\zeta = O\left(\sqrt{\frac{\log(1/\delta)+ \log n}{n}}\right)$ is a lower order term
   and $F(x) = \argmax_{i \in [K]}\Phi_i(x)$.
\end{proposition}

\begin{proposition}[c.f., \cite{wei2020theoretical}, Theorem D.3]
    \label{prop:lw-finite-bd}
    With probability at least $1-\delta$  over the draw of the training data,
    for all $t \in (0, \infty)$, any $\Phi \in \mF$ satisfies the following:
    \begin{equation*}
        L_w(F, \gpl) = \ex[\pwhat]{\bone(m(\Phi, x, \gpl(x)) \le t)} + 
        \otilde\left(\frac{\sum_{i=1}^p \mcop(\mF_i)}{t \sqrt{n}} \right) +
        \zeta,
    \end{equation*}
   where $\zeta = O\left(\sqrt{\frac{\log(1/\delta)+ \log n}{n}}\right)$ is a lower order term
   and $F(x) = \argmax_{i \in [K]}\Phi_i(x)$.
\end{proposition}
\begin{remark}
    Although we have proved Theorem \ref{thm:main-err-bd} following \citep{wei2020theoretical},
    we had to provide our own proof due to some differences in theoretical assumptions
    (e.g., in our case there does not necessarily exist the ground-truth classifier, 
    although they assumed the expansion property for each $P_i$,
    we assume the expansion property for the weighted probability measure).
    On the other hand, the proofs of \citep[Lemma D.6]{wei2020theoretical}
    and \citep[Theorem D.3]{wei2020theoretical} work for any distribution $P$ on $\mX$ and its empirical distribution 
    $\widehat{P}$. 
    Since $\|w\|_1 \modifiedtext{\Pw}(\mS^c_{\mB}(F)) = R_{\mB, w}(F)$ and 
    $\|w\|_1 \modifiedtext{\Pw}(\{x: F(x) \neq \gpl(x)\}) = L_w(F, \gpl)$,
    Proposition \ref{prop:rw-finite-bd} and Proposition \ref{prop:lw-finite-bd} follow from 
    the corresponding results in \cite{wei2020theoretical}.
\end{remark}

\begin{theorem}
    Suppose Assumptions \ref{assump:c-exp}, \ref{assump:pseudo-labeler}, and \ref{assump:cov-cond} hold.
    Then, with probability at least $1 - \delta$ over the draw of the training data,
    for all $t_1, t_2 \in (0, \infty)$, and any neural network $\Phi$ in $\mF$, 
    we have the following:
    \begin{multline*}
        \err_w(F) = 
        \frac{\gamma + 1}{\gamma - 1}
        \ex[\pwhat]{\bone(m(\Phi, x, \gpl(x)) \le t_1)}
        + \frac{2\gamma}{\gamma - 1}\ex[\pwhat]{\bone(m_{\mB}(\Phi, x) \le t_2)} 
        \\
        - \err_w(\gpl) + 2\err_w(\gstar)+
        \frac{2\gamma}{\gamma - 1}R_{\mB, w}(\gstar)\\
        +\otilde\left(\frac{\sum_{i=1}^p \mcop(\mF_i)}{t_1 \sqrt{n}} \right) +
        \otilde\left(\frac{\sum_{i=1}^p \mcop(\mF_i)}{t_2 \sqrt{n}} \right) +
        \zeta,
    \end{multline*}
   where $\zeta = O\left(\sqrt{\frac{\log(1/\delta)+ \log n}{n}}\right)$ is a lower order term
   and $F(x) = \argmax_{i \in [K]}\Phi_i(x)$.
\end{theorem}
\begin{proof}
    By \eqref{eq:errw-lem-applied} with $H = \gstar$ and 
    $-\bone(\gpl(x) \neq \gstar(x)) \le -\bone(\gpl(x) \neq j) + \bone(\gstar(x) \neq j)$
    for any $x \in \mX$ and $j \in [K]$, we obtain the following inequality:
    \begin{equation*}
        \err_w(F) \le \frac{\gamma + 1}{\gamma - 1} L_w(F, \gpl) + \frac{2\gamma}{\gamma - 1} R_{\mB, w}(F)
        + \frac{2\gamma}{\gamma - 1}R_{\mB, w}(\gstar) - \err_w(\gpl) + 2\err_w(\gstar).
    \end{equation*}
    Then, the statement of the theorem follows from Proposition \ref{prop:rw-finite-bd} and Proposition \ref{prop:lw-finite-bd}.
\end{proof}

\section{Proof of Proposition \ref{prop:kl-weighted}}
\label{sec:appendix-proof-prop}
\begin{proof}

Let the average weighted consistency loss be $\mathcal{L}^{wt}_{u} = \frac{1}{|B|}\sum_{x \in B} \ell^{wt}_{u}(\hat{p}_{m}(x),p_{m}(\mathcal{A}(x)), \bG)$ this will be minimized if for each of $x \in B$ the $\ell^{wt}_{u}(\hat{p}_{m}(x),p_{m}(\mathcal{A}(x)), \bG)$ is minimized. This expression can be expanded as:
\begin{align*}
    \ell^{wt}_{u}(\hat{p}_{m}(x), p_{m}(\mathcal{A}(x)), \bG) &= -\sum_{i=1}^{K}(\bG^{T}\hat{p}_{m}(x))_i \log(p_{m}(\mathcal{A}(x))_i) \\
    &= -C\sum_{i=1}^{K}\frac{(\bG^{T}\hat{p}_{m}(x))_i}{\sum_{j = 1}^{m}\bG^{T}\hat{p}_{m}(x))_j} \log(p_{m}(\mathcal{A}(x))_i) \\
    & = C \times \mathrm{H}(\norm(\bG^T\hat{p}_{m}(x)) \; || \; p_{m}(\mathcal{A}(x))).
\end{align*}
Here we use H to denote the cross entropy between two distributions. As we don't backpropogate gradients from the $\hat{p}_{m}(x)$ (pseudo-label) branch of prediction network we can consider $C = \sum_{j = 1}^{m}\bG^{T}\hat{p}_{m}(x))_j$ as a constant in our analysis. Also adding a constant term of entropy $\mathrm{H}(\norm(\bG^{T}\hat{p}_{m}(x)))$ to cross entropy term and dropping constant $C$ doesn't change the outcome of minimization. Hence we have the following:
\begin{align*}
    \min_{p_{m}} \mathrm{H}(\norm(\bG^T\hat{p}_{m}(x)) \; || \; p_{m}(\mathcal{A}(x))) &= \min_{p_{m}} \mathrm{H}(\norm(\bG^T\hat{p}_{m}(x)) \; || \; p_{m}(\mathcal{A}(x))) \\
    & \qquad \qquad + \mathrm{H}(\norm(\bG^{T}\hat{p}_{m}(x))) \\
    &= \min_{p_{m}} \mathcal{D}_{KL}(\norm(\bG^{T}\hat{p}_{m}(x)) || p_{m}(\mathcal{A}(x))).
\end{align*}

This final term is the $\mathcal{D}_{KL}(\bG^{T}\hat{p}_{m}(x) || p_{m}(\mathcal{A}(x))$ which is obtained by using the identity $\mathcal{D}_{KL}(p,q) = \mathrm{H}(p,q) + \mathrm{H}(p)$ where $p, q$ are the two distributions.
\end{proof}

\section{Notation}
We provide the list of notations commonly used in the paper in Table \ref{tab:notations}. 
\begin{table}[htbp]\caption{Table of Notations used in Paper}
\begin{center}%
\begin{tabular}{r c p{10cm} }
\toprule
    \centering
    $\mY$ & $:$ & Label space\\
    $\mX$ & $:$ & Instance space \\
    $K$ & $:$ & Number of classes\\
    $\pi_i$ & $:$ & prior for class i\\  
    $F$ & $:$ & a classifier model\\
    $D$ & $:$ & data distribution \\
    $\blambda$ & $:$ & Lagrange multiplier \\
    $\lambda_u$ & $:$ & coefficient of unlabeled loss \\
    $\text{rec}_i[F]$ & $:$ & recall of $i^{th}$ class for a classifier $F$ \\
    $\text{acc}[F]$ & $:$ & accuracy for a classifier $F$ \\
    $\text{prec}_i[F]$ & $:$ & precision of $i^{th}$ class for a classifier $F$ \\
    $\text{cov}_i[F]$ & $:$ & coverage for $i^{th}$ class for a classifier $F$ \\
    $\text{G}$ & $:$ & a $K \times K$ matrix\\
    $\text{D}$ & $:$ & a $K \times K$ diagonal matrix\\
    $\text{M}$ & $:$ & a $K \times K$ matrix\\
    $\mu$ & $:$ & ratio of labelled to unlabelled samples\\
    $B$ & $:$ & batch size for FixMatch\\
    $\ell_u^{\text{wt}}$ & $:$ & loss for unlabelled data using pseudo label\\
    $\ell_s^{\text{hyb}}$ & $:$ & loss for labelled data\\
    $\mathcal{L}_u^{\text{hyb}}$ & $:$ & average loss for unlabelled data using pseudo label on a batch of samples\\
    $\mathcal{L}_s^{\text{hyb}}$ & $:$ & average loss for labelled data on a batch of samples\\
    $H$ & $:$ & cross entropy function\\
    $\mA$ & $:$ & a $\mX \rightarrow \mX$ function that is stochastic in nature and applies a strong augmentation to it \\
    $\alpha$ & $:$ & a $\mX \rightarrow \mX$ function that is stochastic in nature and applies a weak augmentation to it \\
    $\rho$ & $:$ & imbalance factor \\
    $B$ & $:$ & batch size of samples\\
    $B_s$ & $:$ & batch of labelled samples\\
    $B_u$ & $:$ & batch of unlabelled samples\\
    $x$ & $:$ & an input sample, $x \in \mX$\\
    $\hat{p}_{m}$ & $:$ & a pseudo label generating function \\
    $p_{m}$ & $:$ & distribution of confidence for a model's prediction on a given sample \\
    $w$ & $:$ & a $K \times K$ weight matrix that corresponds to a gain matrix $\bG$\\
    $\err_w(F)$ & $:$ & weighted error of $F$ that corresponds to the objective of CSL\\
    $\addedtext{\mathcal{P}_w}$ & $:$ & weighted distribution on $\mX$\\
    $\addedtext{P_i}$ & $:$ & \addedtext{ class conditional distribution of samples for class $i$}\\
    $R_{\mB, w}(F)$ & $:$ & theoretical weighted (cost sensitive) consistency regularizer\\
    $\gpl$ & $:$ & a pseudo labeler\\
    $L_w(F, \gpl)$ & $:$ & weighted error between $F$ and $\gpl$\\
    $\mL_w(F)$ & $:$ & theoretical \csst{} loss\\
    $c$ & $:$  & a non-increasing function used in the definition of the $c$-expansion property (Definition \ref{def:c-exp-def})\\
    $\gamma$  & $:$ & a value of $c$ defined in Assumption \ref{assump:pseudo-labeler}\\
    $\regub$ & $:$ & an upper bound of $R_{\mB, w}(F)$ in the optimization problem \eqref{eq:cons-cnst-obj}\\
    $S^c$ & $:$ & the complement of a set $S$\\

\bottomrule

\end{tabular}
\end{center}
\label{tab:notations}
\end{table}

\section{Code, License, Assets and Computation Requirements}
\label{sec:code-licence-assets-comp-req}
\subsection{Code and Licenses of Assets}
In this work, we use the open source implementation of FixMatch~\cite{sohn2020fixmatch} \footnote{https://github.com/LeeDoYup/FixMatch-pytorch} in PyTorch, which is licensed under \ttt{MIT} License for educational purpose. Also for NLP experiments we make use of DistillBERT~\cite{sanh2019distilbert} pretrained model available in the HuggingFace~\cite{wolf2020transformers} library. The code to reproduce the main experiments results can be found at \href{https://github.com/val-iisc/CostSensitiveSelfTraining}{https://github.com/val-iisc/CostSensitiveSelfTraining}.

\subsection{Computational Requirements}
\begin{table}[htb]
\centering

\begin{tabular}{cccccccc} 
\toprule
    Method &  CIFAR-10 & CIFAR-100 & ImageNet-100 \\
     \hline
   \texttt{ERM}   &  \begin{tabular}{@{}c@{}}A5000 \\ 49m\end{tabular}  & \begin{tabular}{@{}c@{}}A5000 \\ 6h 47m\end{tabular} & \begin{tabular}{@{}c@{}} RTX3090 \\ 15h 8m\end{tabular} \\
   \texttt{LA}   &  \begin{tabular}{@{}c@{}}RTX3090 \\ 
39m\end{tabular}  & \begin{tabular}{@{}c@{}}A5000 \\ 6h 9m\end{tabular} & \begin{tabular}{@{}c@{}}A5000 \\ 15h 7m \end{tabular} \\
   \texttt{CSL}   &  \begin{tabular}{@{}c@{}}A5000 \\ 47m\end{tabular}  & \begin{tabular}{@{}c@{}}A5000 \\ 6h 40m\end{tabular} & \begin{tabular}{@{}c@{}}A5000 \\ 12h\end{tabular} \\
   \begin{tabular}{@{}c@{}}\texttt{CSST}(FixMatch)\\ w/o KL-Threshold\end{tabular} &  \begin{tabular}{@{}c@{}}4 X A5000 \\ 21h 0m\end{tabular}  & \begin{tabular}{@{}c@{}}4 X A100 \\ 2d 19h 16 m\end{tabular} & \begin{tabular}{@{}c@{}}4 X A5000 \\ 2d 13h 19m\end{tabular} \\
   \texttt{CSST}(FixMatch)    &  \begin{tabular}{@{}c@{}}4 X A5000 \\  21h 41m \end{tabular}  & \begin{tabular}{@{}c@{}}4 X A5000 \\2d 11h 52m \end{tabular} & \begin{tabular}{@{}c@{}}4 X A5000 \\ 2d 4m\end{tabular} \\
    \hline
\end{tabular}
\captionsetup{width=.68\textwidth}
\caption{Computational requirements and training time (d:days, h:hours, m:minutes) for  experiments relevant to vision datasets. As we can see some of the experiments on the larger datasets such as ImageNet requires long compute times.}
\label{tab:GPU-usage-vision}
\end{table}

\begin{table}[htb]
\centering
\begin{tabular}{cccccccc} 
\toprule
    Method &  IMDb($\rho=10$) & IMDb($\rho=100$) & DBpedia-14\\
     \hline
   \texttt{ERM}   &  \begin{tabular}{@{}c@{}}4 X A5000 \\ 25m \end{tabular}  & \begin{tabular}{@{}c@{} }4 X A5000 \\ 29m \end{tabular} & \begin{tabular}{@{}c@{}}4 X A5000 \\2h 44m\end{tabular} \\
   \texttt{UDA}   &  \begin{tabular}{@{}c@{}}4 X A5000 \\ 44m\end{tabular}  & \begin{tabular}{@{}c@{}}4 X A5000 \\ 32m\end{tabular} & \begin{tabular}{@{}c@{}}4 X A5000 \\ 10h 18m\end{tabular} \\
   \texttt{CSST}(UDA)    &  \begin{tabular}{@{}c@{}}4 X A5000 \\ 49m\end{tabular}  & \begin{tabular}{@{}c@{}}4 X A5000 \\ 35m\end{tabular} & \begin{tabular}{@{}c@{}}4 X A5000 \\ 13h 12m\end{tabular} \\
    \hline
\end{tabular}
\captionsetup{width=.68\textwidth}
\caption{Computational requirements and training time(d:days, h:hours, m:minutes)  for experiments done on NLP datasets. The DistilBERT model which we are using is pretrained on a language modeling task, hence it requires much less time for training in comparison to vision models which are trained from scratch.}
\label{tab:GPU-usage-nlp}
\end{table}

All experiments were done on a variety of GPUs, with primarily Nvidia A5000 (24GB) with occasional use of Nvidia A100 (80GB) and Nvidia RTX3090 (24GB). For finetuning DistilBERT and all experiments with ImageNet-100 dataset we used PyTorch data parallel over 4 A5000s. Training was done till no significant change in metrics was observed. The detailed list of computation used per experiment type and dataset have been tabulated in  Table \ref{tab:GPU-usage-vision} and Table \ref{tab:GPU-usage-nlp}.
\section{Dataset}
\textbf{CIFAR-10 and CIFAR-100~\cite{krizhevsky2009learning}.} are image classification datasets of images of size 32 X 32. Both the datasets have a size of 50k samples and by default, they have a uniform sample distribution among its classes. CIFAR-10 has 10 classes while CIFAR-100 has 100 classes. The test set is a balanced set of 10k images.

\textbf{ImageNet-100~\cite{russakovsky2015imagenet}.} is an image classification dataset carved out of ImageNet-1k by selecting the first 100 classes. The distribution of samples is uniform with 1.3k samples per class. The test set contains 50 images per class. All have a resolution of 224X224, the same as the original ImageNet-1k dataset.

\textbf{IMDb\cite{maas-EtAl:2011:ACL-HLT2011}.} dataset is a binary text sentiment classification dataset. The data distribution is uniform by default and has a total 25k samples in both trainset and testset. In this work, we converted the dataset into a longtailed version of $\rho=10, 100$ and selected 1k labeled samples while truncating the labels of the rest and using them as unlabeled samples.

\textbf{DBpedia-14\cite{lehmann2015dbpedia}.} is a topic classification dataset with a uniform distribution of labeled samples. The dataset has 14 classes and has a total of 560k samples in the trainset and 70k samples in the test set. Each sample, apart from the content, also has title of the article that could be used for the task of topic classification. In our experiments, we only make use of the content.

\section{Objective}
\subsection{Logit Adjusted Weighted Consistency Regularizer}
\label{sec:La-hyb-loss}
As we have introduced weighted consistency regularizer in Eq. \eqref{eq:wt-consistency-reg} for utilizing unlabeled data, we now provide  logit adjusted variant of it for training deep networks in this section. We provide logit adjusted term for $ \ell^{\mathrm{wt}}_{u}(\hat{p}_{m}(x), p_{m}(\mathcal{A}(x), \bG)$ below:
\begin{align*}
      \ell^{\mathrm{wt}}_{u}(\hat{p}_{m}(x), p_{m}(\mathcal{A}(x), \bG) &= -\sum_{i=1}^{K}(\bG^{\mathbf{T}}\hat{p}_{m}(x))_i \log(p_{m}(\mathcal{A}(x))_i) \\
      &= -\sum_{i=1}^{K}(\bG^{\mathbf{T}}\hat{p}_{m}(x))_i \log\left(\frac{\exp({p_{m}(\mathcal{A}(x))_i})}{\sum_{j=1}^{K}\exp({p_{m}(\mathcal{A}(x))_j})}\right) \\
      &= -\sum_{i=1}^{K}(\mathrm{\bm{D}}^{\mathbf{T}} \mathrm{\bm{M}}^{\mathbf{T}}\hat{p}_{m}(x))_i \log\left(\frac{\exp({p_{m}(\mathcal{A}(x))_i})}{\sum_{j=1}^{K}\exp({p_{m}(\mathcal{A}(x))_j})}\right) 
\end{align*}
The above expression comes from the decomposition $\bG=\mathrm{MD}$. The above loss function can be converted into it's logit adjusted equivalent variant by following transformation as suggested by \citet{narasimhan2021training} which is equivalent in terms of optimization of deep neural networks:
\begin{align}
      \ell^{\mathrm{wt}}_{u}(\hat{p}_{m}(x), p_{m}(\mathcal{A}(x), \bG)  \equiv -\sum_{i=1}^{K}( \mathrm{\bm{M}}^{\mathbf{T}}\hat{p}_{m}(x))_i \log\left(\frac{\exp({p_{m}(\mathcal{A}(x))_i -\log( \bm{D}_{ii}}))}{\sum_{j=1}^{K}\exp({p_{m}(\mathcal{A}(x))_j  -\log(\bm{D}_{jj}}))}\right)
      \label{eq:lgt-adj-wt-cons}
\end{align}
The above loss is the consistency loss $\ell^{\mathrm{wt}}_{u}$ that we practically implement for \ttt{CSST}. Further in case $\hat{p}_{m}(x)$ is a hard pseudo label $y$ as in FixMatch, the above weighted consistency loss reduces to $\ell^{\mathrm{hyb}}(y, p_{m}(\mathcal{A}(x)))$. Further in case the gain matrix $G$ is diagonal the above loss will converge to $\ell^{\mathrm{LA}}(y, p_{m}(\mathcal{A}(x)))$. Thus the weighted consistency regularizer can be converted to logit adjusted variants $\ell^{\mathrm{LA}}$ and $\ell^{\mathrm{hyb}}$ based on $\bG$ matrix.

\subsection{\texttt{CSST}(FixMatch)}
\label{sec:appendix-fixmatch-obj}
In FixMatch, we use the prediction made by the model on a sample $x$ after applying a weak augmentation $\alpha$ and is used to get a hard pseudo label for the models prediction on a strongly augmented sample i.e. $\mathcal{A}(x)$. The set of weak augmentations include horizontal flip,  We shall refer to this pseudo label as $\hat{p}_{m}(x)$. The list of strong augmentations are given in Table 12 of \citet{sohn2020fixmatch}. Weak augmentations include padding, random horizontal flip and cropping to the desired dimensions (32X32 for CIFAR and 224X224 for ImageNet).
Given a batch of labeled and unlabeled samples $B_s$ and $B_u$, \texttt{CSST} modifies the supervised and un-supervised component of the loss function depending upon the non-decomposable objective and its corresponding gain matrix $\bG$ at a given time during training. We assume that in the dataset, a sample $x$, be it labeled or unlabeled is already weakly augmented.  \texttt{vanilla} FixMatch's supervised componenet of the loss function is a simple cross entropy loss whereas in our \texttt{CSST}(FixMatch) it is replaced by $\ell_s^{\text{hyb}}$ . 
\begin{equation}
    \label{eq:hyb-sup}
    \mathcal{L}_s^{\text{hyb}} =  \frac{1}{|B_s|}\sum_{x,y \in B_s}{\ell^{\text{hyb}}( y, s(x))}.
\end{equation}
\vspace{-3mm}
\begin{equation}
    \mathcal{L}_u^{\mathrm{wt}} = \frac{1}{|B_u|}\sum_{x \in B_u}\indc_{(\mathcal{D}_{KL} (\norm(\bG^{T}\hat{p}_{m}(x))  \; || \; p_m(x))\leq \tau)}
     \ell^{\mathrm{wt}}_{u}(\hat{p}_{m}(x), p_{m}(\mathcal{A}(x)), \bG)).
     \label{loss:supp-hybloss}
\end{equation}

The component of the loss function for unlabeled data (i.e. consistency regularization) is where one of our contributions w.r.t the novel thresholding mechanism comes into light. \texttt{vanilla} FixMatch selects unlabeled samples for which consistency loss is non-zero, such that the model's confidence on the most likely predicted class is above a certain threshold. We rather go for a threshold mechanism that select based on the basis of degree of distribution match to a target distribution based on $\bG$. The final loss function $\mathcal{L} = \mathcal{L}_s^{hyb} + \lambda_u \mathcal{L}_u^{\mathrm{wt}}$, i.e. a linear combination of $\mathcal{L}_s^{\mathrm{hyb}}$ and $\mathcal{L}_u^{wt}$. Since for FixMatch we are dealing with Wide-ResNets and ResNets which are deep networks, as mentioned in Section \ref{sec:La-hyb-loss}, we shall use the alternate logit adjusted formulation as mentioned in Eq. \eqref{eq:lgt-adj-wt-cons} as substitute for $\ell_u^{\mathrm{wt}}$ in Eq. \eqref{loss:supp-hybloss}.

\subsection{\texttt{CSST}(UDA)}
\label{sec:appendix-uda-obj}
The loss function of UDA is a linear combination of supervised loss and consistency loss on unlabeled samples. The former is the cross entropy (CE) loss, while the latter for the unlabeled samples minimizes the KL-divergence between the model's predicted label distribution on an input sample and its augmented sample. Often the predicted label distribution on the unaugmented sample is sharpened. The augmentation we used was a English-French-English backtranslation based on the MarianMT~\cite{junczys-dowmunt-etal-2018-marian} fast neural machine translation model. In UDA supervised component of the loss is annealed using a method described as Training Signal Annealing (TSA), where the CE loss is considered only for those labeled samples whose $\max_i p_m(x)_i < \tau_t$, where $t$ is a training time step. We observed that using TSA in a long tailed setting leads to overfitting on the head classes and hence chose to not include the same in our final implementation.

\texttt{CSST} modifies the supervised and unsupervised component of the loss function in UDA depending upon a given objective and its corresponding gain matrix $\bG$ at a given time during training. The supervised component of the loss function for a given constrained optimization problem and a gain matrix $\bG$, is the hybrid loss $\ell_s^{\text{hyb}}$. For the consistency regularizer part of the loss function, we minimize the KL-divergence between a target distribution and the model's prediction label distribution on its augmented version.
The target distribution is $\norm(\bG^{T}\hat{p}_m(x))$, where $\hat{p}_m(x)$ is the sharpened prediction of the label distribution by the model. Given a batch of labeled and unlabeled samples $B_s$ and $B_u$, the final loss function in \ttt{CSST}(UDA) is a linear combination of $\mathcal{L}_s^{hyb}$ and $\mathcal{L}_u^{wt}$, i.e $\mathcal{L} = \mathcal{L}_s^{hyb} + \lambda_u \mathcal{L}_u^{wt}$. 
\begin{gather}
    \mathcal{L}_s^{\text{hyb}} =  \frac{1}{|B_s|}\sum_{x,y \in B_s}{\ell^{\text{hyb}}(p_m(x), y)}. \\
    \mathcal{L}_u^{wt} = \frac{1}{|B_u|}\sum_{x \in B_u}\indc_{(\mathcal{D}_{KL} (\norm(\bG^{T}\hat{p}_{m}(x))  \; || \; p_m(x))\leq \tau)}
     \ell^{\mathrm{wt}}_{u}(\hat{p}_{m}(x), p_{m}(\mathcal{A}(x), \bG)).
\end{gather}
Since for UDA, we are dealing with DistilBERT, as mentioned in Section \ref{sec:La-hyb-loss}, we shall use the alternate formulation as mentioned in Eq. \eqref{eq:lgt-adj-wt-cons} as substitute for $\ell_u^{\mathrm{wt}}$ in Eq. above. 
\section{Algorithms}
We provide a detailed description of algorithms used for optimizing non decomposable objectives through \ttt{CSST}(FixMatch) ans \texttt{CSST}(UDA). Algorithm \ref{algo:worst-case} is used for experiments in Section \ref{subsec:worst-case-recall} for maximizing worst-case recall (i.e. min recall using \ttt{CSST}(FixMatch) and \ttt{CSST}(UDA)). Algorithm \ref{algo:coverage} is used for experiments in Section \ref{subsec:worst-case-recall} for maximizing recall under coverage constraints (i.e. min coverage experiments on CIFAR10-LT, CIFAR100-LT and ImageNet100-LT).
\begin{algorithm}[H]
\caption{\texttt{CSST}-based Algorithm for Maximizing Worst-case Recall }
\label{algo:worst-case}
\begin{algorithmic}
\STATE Inputs: Training set $S_s$(labeled) and $S_u$(unlabeled) , Validation set $S^{\text{val}}$, Step-size $\omega \in \mathbb{R}_{+}$, Class priors $\pi$
\STATE Initialize: Classifier $F^0$, Multipliers $\blambda^0 \in \Delta_{K -1}$
\FOR{$t = 0 $ to $T-1$}
\STATE \textbf{Update $\blambda$:}~\\[2pt]
\STATE ~~~~$\lambda^{t+1}_i = \lambda^t_i\exp\left(-\omega \cdot \text{recall}_i[F^{t}] \right), \forall i,$\\
\STATE ~~~~$\blambda = \norm (\blambda)$
\STATE ~~~~$\bG =  \diag(\lambda^{t+1}_1/\pi_1, \ldots, \lambda^{t+1}_K/\pi_K)$~\\
\STATE \text{Compute $\ell_u^{\text{wt}}$, $\ell_s^{\text{hyb}}$ using $\bG$ }\\
\STATE \textbf{Cost-sensitive Learning (CSL) for FixMatch:}~\\[2pt]
\STATE  ~~~~$B_u \sim S_u , B_s \sim S_s$ ~// Sample batches of data
\STATE ~~~~$F^{t+1} \,\in\, \arg \min_{F} \sum_{B_u, B_s}\lambda_u\mathcal{L}_{u}^{\mathrm{wt}} + \mathcal{L}_s^{\mathrm{hyb}}$ ~// Replaced by few steps of SGD

\ENDFOR
\RETURN $F^T$
\end{algorithmic}
\end{algorithm}

\begin{algorithm}[H]
\caption{\texttt{CSST}-based Algorithm for Maximizing Mean Recall s.t. per class coverage > 0.95/K }
\label{algo:coverage}
\begin{algorithmic}
\STATE Inputs: Training set $S_s$(labeled) and $S_u$(unlabeled) , Validation set $S^{\text{val}}$, Step-size $\omega \in \mathbb{R}_{+}$, Class priors $\pi$
\STATE Initialize: Classifier $F^0$, Multipliers $\blambda^0 \in \RR_+^{K}$
\FOR{$t = 0 $ to $T-1$}
\STATE \textbf{Update $\blambda$:}~\\[2pt]
\STATE ~~~$\lambda^{t+1}_i = \lambda^t_i - \omega\big(\text{cov}_i[F^{t}] - \frac{0.95}{K}\big), \forall i$\\
\STATE ~~~~$\lambda^{t+1}_i \,=\, \max\{0, \lambda^{t+1}_i\}, \forall i \in [K]$ ~~~// Projection to $\mathbb{R}_+$
\STATE ~~~~$\bG =  \diag(\lambda^{t+1}_1/\pi_1, \ldots, \lambda^{t+1}_K/\pi_K) + \bf{1}_K \blambda^{\trn}$~\\
\STATE \text{Compute $\ell_u^{\text{wt}}$, $\ell_s^{\text{hyb}}$ using $\bG$ }\\

\STATE \textbf{Cost-sensitive Learning (CSL) for FixMatch:}~\\[2pt]
\STATE  ~~~~$B_u \sim S_u , B_s \sim S_s$ ~// Sample batches of data
\STATE ~~~~$F^{t+1} \,\in\, \arg \min_{F} \sum_{B_u, B_s}\lambda_u\mathcal{L}_{u}^{\mathrm{wt}} + \mathcal{L}_s^{\mathrm{hyb}}$ ~// Replaced by few steps of SGD

\ENDFOR
\RETURN $F^T$
\end{algorithmic}
\end{algorithm}

\section{Details of Experiments and Hyper-parameters}
\label{sec:appendix-details-of-experiments}
The experiment of $\max_F \min_i \text{recall}_i[f]$ and $\max_F \text{recall}[F] \text{ s.t. } \text{cov}_i[F] > \frac{0.95}{K}, \forall i \in [K]$  was performed on the long tailed version of CIFAR-10, IMDb($\rho=10,100$) and DBpedia-14 datasets. This was because the optimization of the aforementioned 2 objectives is stable for cases with low number of classes. Hence the objective of $\max_F \min(\text{recall}_{\mH}[F], \text{ recall}_{\mT}[F])$  and $\max_F \text{recall}[F] \text{ s.t. } \min_{\mH, \mT} \text{cov}_{\mH,\mT}[F] > \frac{0.95}{K}$ is a relatively easier objective for datasets with large number of classes, hence were the optimization objectives for CIFAR-100 and ImageNet-100 long tailed datasets. For all experiments for a given dataset, we used the same values for a given common hyperparameter. We ablated the threshold for our novel unlabeled sample selection criterion($\tau$) and the ratio of labeled and unlabeled samples, given fixed number of unlabeled samples($\mu$) and are available in Fig. \textcolor{red}{4b}.

\begin{table}[t]
\centering
\begin{adjustbox}{max width=\textwidth}
\begin{tabular}{cccccccc} 
\toprule
    Parameter &  CIFAR-10 & CIFAR-100 & ImageNet-100 & \begin{tabular}{@{}c@{}}IMDb \\ ($\rho=10$)\end{tabular} & \begin{tabular}{@{}c@{}}IMDb \\ ($\rho=100$)\end{tabular}  &  DBpedia-14\\
     \hline
    $\tau$ &  0.05 & 0.05 & 0.05 & 0.1 & 0.1 & 0.1\\
    $\lambda_u$ & 1.0 & 1.0 & 1.0 & 0.1 & 0.1 & 0.1\\
    $\mu$ & 4.0 & 4.0 & 4.0 & 13.8 & 12.6 & 133\\
    $|B_s|$ & 64 & 64 & 64 & 32 & 32 & 32\\
    $|B_u|$ & 256 & 256 & 256 & 128 & 128 & 128\\
    lr & 3e-3 & 3e-3 & 0.1 & 1e-5 & 1e-5 & 1e-5\\
    $\omega$ & 0.25 & 0.25 & 0.1 & 0.5 & 0.5 & 0.5\\
    \begin{tabular}{@{}c@{}}SGD steps \\ before eval\end{tabular} &  32 & 100 & 500 & 50 & 50 & 100\\
    optimizer & SGD & SGD & SGD & AdamW & AdamW & AdamW \\
    KL-Thresh & 0.95 & 0.95 & 0.95 & 0.9 & 0.9 & 0.9 \\
    Weight Decay & 1e-4 & 1e-3 & 1e-4 & 1e-2 & 1e-2 & 1e-2\\
    $\rho$ &  100 & 10 & 10 & 10 & 100 & 100\\
    $\lambda_u$ &  1.0 & 1.0 & 1.0 & 0.1 & 0.1 & 0.1 \\
    $\mu$ &  4.0 & 4.0 & 4.0 & 11.0 & 11.0 & 11.0 \\
    Arch. &  WRN-28-2 & WRN-28-8 & ResNet50 & DistilBERT & DistilBERT & DistilBERT \\
    \hline
\end{tabular}
\end{adjustbox}
\caption{This table shows us the detailed hyper parameters used for \texttt{CSST}(FixMatch) for the long tailed datasets CIFAR-10, CIFAR-100, ImageNet-100 and \ttt{CSST}(UDA) on IMDb, DBpedia-14. All the datasets were converted to their respective long tailed versions based on the imbalance factor $\rho$, and a fraction of the samples were used along with their labels for supervision.}
\label{tab:hyperparams}
\end{table}
\vspace{-3mm}

\section{\addedtext{Threshold mechanism for diagonal Gain Matrix }}
\label{Diagonal-G-hard-PL}
\addedtext{Consider the case when the gain matrix is a diagonal matrix. The loss function $\mathcal{L}_u^{wt}(B_u)$ as defined in \eqref{loss:CSST-KL} makes uses of a threshold function that selects samples based on the KL divergence based threshold between the target distribution as defined by the gain matrix $\bG$ and the models predicted distribution of confidence over the classes.
\begin{equation}
    \text{Threshold function }\coloneqq \indc_{(\mathcal{D}_{KL} (\norm(\bG^{T}\hat{p}_{m}(x))  \; || \; p_m(x))\leq \tau)}
\end{equation}
Since $\bG$ is a diagonal matrix and the pseudo-label $\hat{p}_m(x)$ is one hot, the $\text{norm}(\bG^T\hat{p}_m(x))$ is a one-hot vector. The threshold function's KL divergence based criterion can be expanded as follows where $\hat{y}$ is the pseudo-label's maximum class's index:
\begin{equation}
    \mathcal{D}_{KL}(\text{norm}(\bG^T\hat{p}_m(x)) || p_{m}(x))   = - \log{p_m(x)}_{\hat{y}} < \tau \\  
\end{equation}
The above equations represents a threshold on the negative log-confidence of the model's prediction for a given unlabeled sample, for the pseudo-label class ($\hat{y}$). This can be further simplified to $p_m(x)_{\hat{y}} \geq \exp{(- \tau)} $ which is simply a threshold based on the model's confidence. Since pseudo-label is generated from the model's prediction, this threshold is nothing but a selection criterion to select only those samples whose maximum confidence for a predicted hard pseudo-label is above a fixed threshold. This is identical to the threshold function which is used in Fixmatch~\cite{sohn2020fixmatch} i.e. $\max(p_m(x)) \geq \exp(-\tau)$. In FixMatch this $\exp(-\tau)$ is set to 0.95.}

\begin{table}[!t]
    \centering
    \small
  
    \captionsetup{width=.55\textwidth}
    \caption{Avg. and std. deviation of Mean Recall and Min. Recall for CIFAR-10 LT}
      \label{tab:statistical-analysis}
    \begin{tabular}{lcccccc}
        \hline
        \multicolumn{1}{c}{\textbf{Method}} & 
        \multicolumn{1}{c}{\textbf{Mean Recall}} & 
        \multicolumn{1}{c}{\textbf{Min Recall}} &
        \\
        \hline
        \texttt{ERM} &  
            0.52 $\pm$ 0.01 &	0.27 $\pm$ 0.02
            \\
        \texttt{LA} &  
            0.54 $\pm$ 0.02 &	0.37 $\pm$ 0.01 
            \\
        \texttt{CSL} &  
            0.63 $\pm$ 0.01 &	0.43 $\pm$ 0.04 
            \\
        \hline
        \begin{tabular}{@{}c@{}}\texttt{Vanilla} (FixMatch)
        \end{tabular}
            & 0.78 $\pm$ 0.01 & 0.47 $\pm$ 0.02 
            \\
        \begin{tabular}{@{}c@{}}\texttt{CSST}(FixMatch) 
        \end{tabular}
            & 0.75 $\pm$ 0.01 &	0.72 $\pm$ 0.01 
            \\
        \hline
    \end{tabular}
    \vspace{0.5em}

    \vspace{-5mm}
\end{table}
\section{Statistical Analysis}
We establish the statistical soundness and validity of our results we ran our experiments on 3 different seeds. Due to the computational requirements for some of the experiments ($\approx $2days) we chose to run the experiments on multiple seeds for a subset of tasks i.e. for maximising the minimum recall among all classes for CIFAR-10 LT. We observe that the std. deviation is significantly smaller than the average values for mean recall and min. recall and our performance metrics fall within our std. deviation hence validating the stability and soundness of training.

\section{Additional Experiments}
\label{sec: additional-experiments}
In this section, we compare \texttt{CSST}(FixMatch) against contemporary semi-supervised learning techniques namely CReST~\cite{Wei_2021_CVPR}, DARP~\cite{10.5555/3495724.3496945},  \texttt{vanilla}(FixMatch)~\cite{sohn2020fixmatch} and FlexMatch~\cite{NEURIPS2021_995693c1}. We compared these methods on the long-tailed CIFAR-10 ($\rho$ = 100) and CIFAR-100 ($\rho$ = 10) datasets. The objective for long-tailed CIFAR-10 dataset was to maximise the worst-case recall \eqref{eq:min-recall-obj} and average recall, subject to a per-class coverage constraint \eqref{eq: cov-const-obj}.  For CIFAR-100  LT dataset, we compare these methods for the objectives maximizing HT recall \eqref{eq:min-HT-recall-obj} and recall under HT coverage constraints\eqref{eq: HT-cov-const-obj}. For the objectives \eqref{eq:min-recall-obj} and  \eqref{eq:min-HT-recall-obj}, DARP achieves the best average recall yet it suffers on the worst-case recall. \texttt{CSST}(FixMatch) outperforms it on CIFAR-100 and  the has superior worst-case recall on CIFAR-10. A similar observation can be made for the objectives \eqref{eq: cov-const-obj} and \eqref{eq: HT-cov-const-obj} where it is only \texttt{CSST}(FixMatch) that has the highest min. coverage and either has superior mean recall or has negligible decrease in the average recall.
\begin{table}[!t]
    \centering
    \small
    \caption{Comparing CSST(FixMatch) against other Semi-Supervised Learning Methods for long tailed data distribution for the objectives \eqref{eq:min-recall-obj} and \eqref{eq:min-HT-recall-obj}. Although \texttt{CSST}(FixMatch) does not achieve the highest mean recall, it at very little cost to mean recall, achieves the best worst-case recall              }
    \begin{tabular}{lcc|cccccc}
        \hline
        \multicolumn{1}{c}{\textbf{Method}} & 
        \multicolumn{2}{c}{\textbf{CIFAR10-LT }} & 
        \multicolumn{2}{c}{\textbf{CIFAR100-LT }} &
        \\
        & \multicolumn{2}{c}{($\rho = 100$)} & 
        \multicolumn{2}{c}{($\rho = 10$)} &
        \\
        \hline
        & \textbf{Avg. Rec} & \textbf{Min Rec} & 
        \textbf{Avg. Rec } & \textbf{Min HT Rec} &
        \\
        \hline
        CReST &	0.72 &	0.47 &	0.52 &	0.46 \\
        DARP  &	\highlight{0.81} &	0.64 &	0.55 &	0.54 \\
        FlexMatch &	0.80 &	0.48 &	0.61 &	0.39 \\
        \begin{tabular}{@{}c@{}}\texttt{Vanilla} (FixMatch)
        \end{tabular}
            & 0.78 &	0.48 &
            \highlight{0.63} &	0.36 \\
        \begin{tabular}{@{}c@{}}\texttt{CSST}(FixMatch) 
        \end{tabular}
            & 0.76 &	\highlight{0.72} &
            \highlight{0.63} &	\highlight{0.61} \\
        
        \hline
    \end{tabular}
    \vspace{0.5em}

    \label{tab:add_expts_min_rec}
    \vspace{-5mm}
\end{table}

\begin{table}[!t]
    \centering
    \small
    \caption{Comparing CSST(FixMatch) against other Semi-Supervised Learning Methods for long tailed data distribution for the objectives \eqref{eq: cov-const-obj}, \eqref{eq: HT-cov-const-obj}. Only \texttt{CSST}(FixMatch) is the closest to satisfying the constraint yet suffers very little on the avg. recall. }
    \begin{tabular}{lcc|cccccc}
        \hline
        \multicolumn{1}{c}{\textbf{Method}} & 
        \multicolumn{2}{c}{\textbf{CIFAR10-LT }} & 
        \multicolumn{2}{c}{\textbf{CIFAR100-LT }} &
        \\
        & \multicolumn{2}{c}{($\rho = 100$, tgt : 0.1)} & 
        \multicolumn{2}{c}{($\rho = 10$, tgt : 0.01)} &
        \\
        \hline
        & \textbf{Avg. Rec} & \textbf{Min Cov} & 
        \textbf{Avg. Rec } & \textbf{Min HT Cov} &
        \\
        \hline
        CReST &	0.72 &	0.052 &	0.52 &	0.009 \\
        DARP  &	\highlight{0.81} &	0.063 &	0.55 &	0.006 \\
        FlexMatch &	0.80 &	0.046 &	0.61 & 	0.006 \\
        \begin{tabular}{@{}c@{}}\texttt{Vanilla} (FixMatch)
        \end{tabular}
            & 0.78 &	0.055 &
            \highlight{0.63} &	0.004 \\
        \begin{tabular}{@{}c@{}}\texttt{CSST}(FixMatch) 
        \end{tabular}
            & 0.80 &	\highlight{0.092} &	\highlight{0.63} &	\highlight{0.010} \\
        
        \hline
    \end{tabular}
    \vspace{0.5em}

    \label{tab:add_expts_min_cov}
    \vspace{-5mm}
\end{table}
\section{\addedtext{Additional Details}}
\subsection{\addedtext{Formal Statement Omitted in Sec. \ref{sec:loss-fn-for-ndo}}}
\label{sec:appendix-formal-statement}
\addedtext{
In Sec. \ref{sec:loss-fn-for-ndo}, 
we stated that learning with the hybrid loss $\hybloss$ gives the Bayes-optimal classifier
for the CSL \eqref{eq:csl-obj}.
However, due to space constraint, we did not provide a formal statement.
In this section, we provide a formal statement of it for clarity.
}
\begin{prop}[\addedtext{\cite{narasimhan2021training} Proposition 4}]
    \addedtext{For any diagonal matrix $\bD \in \RR^{K \times K}$ with $D_{ii} > 0, \forall i$, 
    $\bM \in \RR^{K \times K}$,
    and $\bG =  \bM \bD$, the hybrid loss $\hybloss$ is calibrated for $\bG$.
    That is, for any model-prediction $\widehat{p}_m: \mX \rightarrow \RR^K$ that minimizes 
    $\ex[(x, y) \sim D]{\hybloss(y, \widehat{p}_m(x))}$, the associated classifier $F(x) = \argmax_{y \in [K]} 
    \left(\widehat{p}_{m}\right)_i(x)$ is the Bayes-optimal classifier for CSL \eqref{eq:csl-obj}.}
\end{prop}
\subsection{\addedtext{Comparison with the $(a, \widetilde{c})$-expansion Property in \cite{wei2020theoretical}}}
\label{sec:appendix-a-c-expansion}
\addedtext{We compare the $c$-expansion property with $(a, \widetilde{c})$-expansion property proposed by \cite{wei2020theoretical},
where $a \in (0, 1)$ and $\widetilde{c} > 1$.
Here we say a distribution $Q$ on $\mX$ satisfies the $(a, \widetilde{c})$-expansion property 
if $Q(\mN(S)) \ge \widetilde{c}$ for any $S \subset \mX$ with $Q(S) \le a$.
If $Q$ satisfies $(a, \widetilde{c})$-expansion property \cite{wei2020theoretical} 
with $\widetilde{c} > 1$, then
$Q$ satisfies the $c$-expansion property, where 
the function $c$ is defined as follows.
$c(p) = \widetilde{c}$ if $p \le a$ and $c(p) = 1$ otherwise.
On the other hand, if $Q$-satisfies $c$-expansion property, 
then for any $a \in (0, 1)$ and $S \subseteq \mX$ with $Q(S) \le a$,
we have $Q(\mN(S)) \ge c(Q(S)) Q(S) \ge c(a) Q(S)$ since $c$ is non-increasing.
Therefore, $Q$ satisfies the $(a, c(a))$-expansion property.
Thus, we could say these two conditions are equivalent.
To simplify our analysis, we use our definition of the expansion property.
}

\addedtext{In addition, \citet{wei2020theoretical} showed that the $(a, \widetilde{c})$-expansion property is realistic for vision. Although they assumed the $(a, c)$-expansion property for each $P_i$ $(1 \le i \le K)$ and we assume the $c$-expansion property for $\Pw$, it follows that the $c$-expansion property for $\Pw$ is also realistic for vision, since $\Pw$ is a linear combination of $P_i$}.

\bibliography{mybibliography}

\begin{thebibliography}{48}
\providecommand{\natexlab}[1]{#1}
\providecommand{\url}[1]{\texttt{#1}}
\expandafter\ifx\csname urlstyle\endcsname\relax
  \providecommand{\doi}[1]{doi: #1}\else
  \providecommand{\doi}{doi: \begingroup \urlstyle{rm}\Url}\fi

\bibitem[Berthelot et~al.(2019)Berthelot, Carlini, Goodfellow, Papernot,
  Oliver, and Raffel]{berthelot2019mixmatch}
David Berthelot, Nicholas Carlini, Ian Goodfellow, Nicolas Papernot, Avital
  Oliver, and Colin~A Raffel.
\newblock Mixmatch: A holistic approach to semi-supervised learning.
\newblock In H.~Wallach, H.~Larochelle, A.~Beygelzimer, F.~d\textquotesingle
  Alch\'{e}-Buc, E.~Fox, and R.~Garnett, editors, \emph{Advances in Neural
  Information Processing Systems}. Curran Associates, Inc., 2019.

\bibitem[Bobkov(1997)]{bobkov1997isoperimetric}
Sergey~G Bobkov.
\newblock An isoperimetric inequality on the discrete cube, and an elementary
  proof of the isoperimetric inequality in gauss space.
\newblock \emph{The Annals of Probability}, 25\penalty0 (1):\penalty0 206--214,
  1997.

\bibitem[Cao et~al.(2019)Cao, Wei, Gaidon, Arechiga, and Ma]{cao2019learning}
Kaidi Cao, Colin Wei, Adrien Gaidon, Nikos Arechiga, and Tengyu Ma.
\newblock Learning imbalanced datasets with label-distribution-aware margin
  loss.
\newblock \emph{Advances in neural information processing systems}, 32, 2019.

\bibitem[Chapelle et~al.(2009)Chapelle, Scholkopf, and Zien]{chapelle2009semi}
Olivier Chapelle, Bernhard Scholkopf, and Alexander Zien.
\newblock Semi-supervised learning (chapelle, o. et al., eds.; 2006)[book
  reviews].
\newblock \emph{IEEE Transactions on Neural Networks}, 20\penalty0
  (3):\penalty0 542--542, 2009.

\bibitem[Chiani et~al.(2003)Chiani, Dardari, and Simon]{chiani2003new}
Marco Chiani, Davide Dardari, and Marvin~K Simon.
\newblock New exponential bounds and approximations for the computation of
  error probability in fading channels.
\newblock \emph{IEEE Transactions on Wireless Communications}, 2\penalty0
  (4):\penalty0 840--845, 2003.

\bibitem[Cotter et~al.(2019)Cotter, Jiang, Gupta, Wang, Narayan, You, and
  Sridharan]{cotter2019optimization}
Andrew Cotter, Heinrich Jiang, Maya~R Gupta, Serena Wang, Taman Narayan,
  Seungil You, and Karthik Sridharan.
\newblock Optimization with non-differentiable constraints with applications to
  fairness, recall, churn, and other goals.
\newblock \emph{J. Mach. Learn. Res.}, 20\penalty0 (172):\penalty0 1--59, 2019.

\bibitem[Cui et~al.(2019)Cui, Jia, Lin, Song, and Belongie]{cui2019class}
Yin Cui, Menglin Jia, Tsung-Yi Lin, Yang Song, and Serge Belongie.
\newblock Class-balanced loss based on effective number of samples.
\newblock In \emph{Proceedings of the IEEE/CVF conference on computer vision
  and pattern recognition}, pages 9268--9277, 2019.

\bibitem[Goh et~al.(2016)Goh, Cotter, Gupta, and
  Friedlander]{goh2016satisfying}
Gabriel Goh, Andrew Cotter, Maya Gupta, and Michael~P Friedlander.
\newblock Satisfying real-world goals with dataset constraints.
\newblock \emph{Advances in Neural Information Processing Systems}, 29, 2016.

\bibitem[He et~al.(2016)He, Zhang, Ren, and Sun]{he2016deep}
Kaiming He, Xiangyu Zhang, Shaoqing Ren, and Jian Sun.
\newblock Deep residual learning for image recognition.
\newblock In \emph{Proceedings of the IEEE conference on computer vision and
  pattern recognition}, pages 770--778, 2016.

\bibitem[Huang et~al.(2019)Huang, Zhai, Talbott, Martin, Sun, Guestrin, and
  Susskind]{huang2019addressing}
Chen Huang, Shuangfei Zhai, Walter Talbott, Miguel~Bautista Martin, Shih-Yu
  Sun, Carlos Guestrin, and Josh Susskind.
\newblock Addressing the loss-metric mismatch with adaptive loss alignment.
\newblock In \emph{International conference on machine learning}, pages
  2891--2900. PMLR, 2019.

\bibitem[Junczys-Dowmunt et~al.(2018)Junczys-Dowmunt, Grundkiewicz, Dwojak,
  Hoang, Heafield, Neckermann, Seide, Germann, Aji, Bogoychev, Martins, and
  Birch]{junczys-dowmunt-etal-2018-marian}
Marcin Junczys-Dowmunt, Roman Grundkiewicz, Tomasz Dwojak, Hieu Hoang, Kenneth
  Heafield, Tom Neckermann, Frank Seide, Ulrich Germann, Alham~Fikri Aji,
  Nikolay Bogoychev, Andr{\'e} F.~T. Martins, and Alexandra Birch.
\newblock {M}arian: Fast neural machine translation in {C}++.
\newblock In \emph{Proceedings of {ACL} 2018, System Demonstrations}, pages
  116--121, Melbourne, Australia, July 2018. Association for Computational
  Linguistics.
\newblock \doi{10.18653/v1/P18-4020}.

\bibitem[Kim et~al.(2020)Kim, Hur, Park, Yang, Hwang, and
  Shin]{10.5555/3495724.3496945}
Jaehyung Kim, Youngbum Hur, Sejun Park, Eunho Yang, Sung~Ju Hwang, and Jinwoo
  Shin.
\newblock Distribution aligning refinery of pseudo-label for imbalanced
  semi-supervised learning.
\newblock In \emph{Proceedings of the 34th International Conference on Neural
  Information Processing Systems}, NIPS'20, Red Hook, NY, USA, 2020. Curran
  Associates Inc.
\newblock ISBN 9781713829546.

\bibitem[Kingma et~al.(2014)Kingma, Mohamed, Jimenez~Rezende, and
  Welling]{kingma2014semi}
Durk~P Kingma, Shakir Mohamed, Danilo Jimenez~Rezende, and Max Welling.
\newblock Semi-supervised learning with deep generative models.
\newblock \emph{Advances in neural information processing systems}, 27, 2014.

\bibitem[Krishna et~al.(2017)Krishna, Zhu, Groth, Johnson, Hata, Kravitz, Chen,
  Kalantidis, Li, Shamma, et~al.]{krishna2017visual}
Ranjay Krishna, Yuke Zhu, Oliver Groth, Justin Johnson, Kenji Hata, Joshua
  Kravitz, Stephanie Chen, Yannis Kalantidis, Li-Jia Li, David~A Shamma, et~al.
\newblock Visual genome: Connecting language and vision using crowdsourced
  dense image annotations.
\newblock \emph{International journal of computer vision}, 123\penalty0
  (1):\penalty0 32--73, 2017.

\bibitem[Krizhevsky and Hinton(2009)]{krizhevsky2009learning}
Alex Krizhevsky and Geoffrey Hinton.
\newblock Learning multiple layers of features from tiny images.
\newblock 2009.
\newblock Technical report, University of Toronto.

\bibitem[Kumar et~al.(2021)Kumar, Narasimhan, and Cotter]{kumar2021implicit}
Abhishek Kumar, Harikrishna Narasimhan, and Andrew Cotter.
\newblock Implicit rate-constrained optimization of non-decomposable
  objectives.
\newblock In \emph{International Conference on Machine Learning}, pages
  5861--5871. PMLR, 2021.

\bibitem[Laine and Aila(2016)]{laine2016temporal}
Samuli Laine and Timo Aila.
\newblock Temporal ensembling for semi-supervised learning.
\newblock \emph{arXiv preprint arXiv:1610.02242}, 2016.

\bibitem[Lehmann et~al.(2015)Lehmann, Isele, Jakob, Jentzsch, Kontokostas,
  Mendes, Hellmann, Morsey, Van~Kleef, Auer, et~al.]{lehmann2015dbpedia}
Jens Lehmann, Robert Isele, Max Jakob, Anja Jentzsch, Dimitris Kontokostas,
  Pablo~N Mendes, Sebastian Hellmann, Mohamed Morsey, Patrick Van~Kleef,
  S{\"o}ren Auer, et~al.
\newblock Dbpedia--a large-scale, multilingual knowledge base extracted from
  wikipedia.
\newblock \emph{Semantic web}, 6\penalty0 (2):\penalty0 167--195, 2015.

\bibitem[Lin et~al.(2002)Lin, Lee, and Wahba]{lin2002support}
Yi~Lin, Yoonkyung Lee, and Grace Wahba.
\newblock Support vector machines for classification in nonstandard situations.
\newblock \emph{Machine learning}, 46\penalty0 (1):\penalty0 191--202, 2002.

\bibitem[Maas et~al.(2011)Maas, Daly, Pham, Huang, Ng, and
  Potts]{maas-EtAl:2011:ACL-HLT2011}
Andrew~L. Maas, Raymond~E. Daly, Peter~T. Pham, Dan Huang, Andrew~Y. Ng, and
  Christopher Potts.
\newblock Learning word vectors for sentiment analysis.
\newblock In \emph{Proceedings of the 49th Annual Meeting of the Association
  for Computational Linguistics: Human Language Technologies}, pages 142--150,
  Portland, Oregon, USA, June 2011. Association for Computational Linguistics.

\bibitem[Menon et~al.(2020)Menon, Jayasumana, Rawat, Jain, Veit, and
  Kumar]{menon2020long}
Aditya~Krishna Menon, Sadeep Jayasumana, Ankit~Singh Rawat, Himanshu Jain,
  Andreas Veit, and Sanjiv Kumar.
\newblock Long-tail learning via logit adjustment.
\newblock In \emph{International Conference on Learning Representations}, 2020.

\bibitem[Miyato et~al.(2018)Miyato, Maeda, Koyama, and
  Ishii]{miyato2018virtual}
Takeru Miyato, Shin-ichi Maeda, Masanori Koyama, and Shin Ishii.
\newblock Virtual adversarial training: a regularization method for supervised
  and semi-supervised learning.
\newblock \emph{IEEE transactions on pattern analysis and machine
  intelligence}, 41\penalty0 (8):\penalty0 1979--1993, 2018.

\bibitem[Mohri et~al.(2019)Mohri, Sivek, and Suresh]{mohri2019agnostic}
Mehryar Mohri, Gary Sivek, and Ananda~Theertha Suresh.
\newblock Agnostic federated learning.
\newblock In \emph{International Conference on Machine Learning}, pages
  4615--4625. PMLR, 2019.

\bibitem[Narasimhan and Menon(2021)]{narasimhan2021training}
Harikrishna Narasimhan and Aditya~K Menon.
\newblock Training over-parameterized models with non-decomposable objectives.
\newblock \emph{Advances in Neural Information Processing Systems}, 34, 2021.

\bibitem[Narasimhan et~al.(2014)Narasimhan, Vaish, and
  Agarwal]{narasimhan2014statistical}
Harikrishna Narasimhan, Rohit Vaish, and Shivani Agarwal.
\newblock On the statistical consistency of plug-in classifiers for
  non-decomposable performance measures.
\newblock \emph{Advances in neural information processing systems}, 27, 2014.

\bibitem[Narasimhan et~al.(2015)Narasimhan, Ramaswamy, Saha, and
  Agarwal]{narasimhan2015consistent}
Harikrishna Narasimhan, Harish Ramaswamy, Aadirupa Saha, and Shivani Agarwal.
\newblock Consistent multiclass algorithms for complex performance measures.
\newblock In \emph{International Conference on Machine Learning}, pages
  2398--2407. PMLR, 2015.

\bibitem[Natarajan et~al.(2016)Natarajan, Koyejo, Ravikumar, and
  Dhillon]{natarajan2016optimal}
Nagarajan Natarajan, Oluwasanmi Koyejo, Pradeep Ravikumar, and Inderjit
  Dhillon.
\newblock Optimal classification with multivariate losses.
\newblock In \emph{International Conference on Machine Learning}, pages
  1530--1538. PMLR, 2016.

\bibitem[Patrini et~al.(2017)Patrini, Rozza, Krishna~Menon, Nock, and
  Qu]{patrini2017making}
Giorgio Patrini, Alessandro Rozza, Aditya Krishna~Menon, Richard Nock, and
  Lizhen Qu.
\newblock Making deep neural networks robust to label noise: A loss correction
  approach.
\newblock In \emph{Proceedings of the IEEE conference on computer vision and
  pattern recognition}, pages 1944--1952, 2017.

\bibitem[Pham et~al.(2021)Pham, Dai, Xie, and Le]{pham2021meta}
Hieu Pham, Zihang Dai, Qizhe Xie, and Quoc~V Le.
\newblock Meta pseudo labels.
\newblock In \emph{Proceedings of the IEEE/CVF Conference on Computer Vision
  and Pattern Recognition}, pages 11557--11568, 2021.

\bibitem[Puthiya~Parambath et~al.(2014)Puthiya~Parambath, Usunier, and
  Grandvalet]{puthiya2014optimizing}
Shameem Puthiya~Parambath, Nicolas Usunier, and Yves Grandvalet.
\newblock Optimizing f-measures by cost-sensitive classification.
\newblock \emph{Advances in neural information processing systems}, 27, 2014.

\bibitem[Russakovsky et~al.(2015)Russakovsky, Deng, Su, Krause, Satheesh, Ma,
  Huang, Karpathy, Khosla, Bernstein, et~al.]{russakovsky2015imagenet}
Olga Russakovsky, Jia Deng, Hao Su, Jonathan Krause, Sanjeev Satheesh, Sean Ma,
  Zhiheng Huang, Andrej Karpathy, Aditya Khosla, Michael Bernstein, et~al.
\newblock Imagenet large scale visual recognition challenge.
\newblock \emph{International journal of computer vision}, 115\penalty0
  (3):\penalty0 211--252, 2015.

\bibitem[Saito et~al.(2017)Saito, Ushiku, and Harada]{saito2017asymmetric}
Kuniaki Saito, Yoshitaka Ushiku, and Tatsuya Harada.
\newblock Asymmetric tri-training for unsupervised domain adaptation.
\newblock In \emph{International Conference on Machine Learning}, pages
  2988--2997. PMLR, 2017.

\bibitem[Sanh et~al.(2019)Sanh, Debut, Chaumond, and Wolf]{sanh2019distilbert}
Victor Sanh, Lysandre Debut, Julien Chaumond, and Thomas Wolf.
\newblock Distilbert, a distilled version of bert: smaller, faster, cheaper and
  lighter.
\newblock \emph{arXiv preprint arXiv:1910.01108}, 2019.

\bibitem[Sanyal et~al.(2018)Sanyal, Kumar, Kar, Chawla, and
  Sebastiani]{sanyal2018optimizing}
Amartya Sanyal, Pawan Kumar, Purushottam Kar, Sanjay Chawla, and Fabrizio
  Sebastiani.
\newblock Optimizing non-decomposable measures with deep networks.
\newblock \emph{Machine Learning}, 107\penalty0 (8):\penalty0 1597--1620, 2018.

\bibitem[Sohn et~al.(2020)Sohn, Berthelot, Carlini, Zhang, Zhang, Raffel,
  Cubuk, Kurakin, and Li]{sohn2020fixmatch}
Kihyuk Sohn, David Berthelot, Nicholas Carlini, Zizhao Zhang, Han Zhang,
  Colin~A Raffel, Ekin~Dogus Cubuk, Alexey Kurakin, and Chun-Liang Li.
\newblock Fixmatch: Simplifying semi-supervised learning with consistency and
  confidence.
\newblock \emph{Advances in Neural Information Processing Systems},
  33:\penalty0 596--608, 2020.

\bibitem[Tavker et~al.(2020)Tavker, Ramaswamy, and
  Narasimhan]{tavker2020consistent}
Shiv~Kumar Tavker, Harish~Guruprasad Ramaswamy, and Harikrishna Narasimhan.
\newblock Consistent plug-in classifiers for complex objectives and
  constraints.
\newblock \emph{Advances in Neural Information Processing Systems},
  33:\penalty0 20366--20377, 2020.

\bibitem[Thomee et~al.(2016)Thomee, Shamma, Friedland, Elizalde, Ni, Poland,
  Borth, and Li]{thomee2016yfcc100m}
Bart Thomee, David~A Shamma, Gerald Friedland, Benjamin Elizalde, Karl Ni,
  Douglas Poland, Damian Borth, and Li-Jia Li.
\newblock Yfcc100m: The new data in multimedia research.
\newblock \emph{Communications of the ACM}, 59\penalty0 (2):\penalty0 64--73,
  2016.

\bibitem[Wei et~al.(2021)Wei, Sohn, Mellina, Yuille, and Yang]{Wei_2021_CVPR}
Chen Wei, Kihyuk Sohn, Clayton Mellina, Alan Yuille, and Fan Yang.
\newblock Crest: A class-rebalancing self-training framework for imbalanced
  semi-supervised learning.
\newblock In \emph{Proceedings of the IEEE/CVF Conference on Computer Vision
  and Pattern Recognition (CVPR)}, pages 10857--10866, June 2021.

\bibitem[Wei and Ma(2019)]{wei2019improved}
Colin Wei and Tengyu Ma.
\newblock Improved sample complexities for deep neural networks and robust
  classification via an all-layer margin.
\newblock In \emph{International Conference on Learning Representations}, 2019.

\bibitem[Wei et~al.(2020)Wei, Shen, Chen, and Ma]{wei2020theoretical}
Colin Wei, Kendrick Shen, Yining Chen, and Tengyu Ma.
\newblock Theoretical analysis of self-training with deep networks on unlabeled
  data.
\newblock In \emph{International Conference on Learning Representations}, 2020.

\bibitem[Wolf et~al.(2020)Wolf, Debut, Sanh, Chaumond, Delangue, Moi, Cistac,
  Rault, Louf, Funtowicz, et~al.]{wolf2020transformers}
Thomas Wolf, Lysandre Debut, Victor Sanh, Julien Chaumond, Clement Delangue,
  Anthony Moi, Pierric Cistac, Tim Rault, R{\'e}mi Louf, Morgan Funtowicz,
  et~al.
\newblock Transformers: State-of-the-art natural language processing.
\newblock In \emph{Proceedings of the 2020 conference on empirical methods in
  natural language processing: system demonstrations}, pages 38--45, 2020.

\bibitem[Xie et~al.(2020{\natexlab{a}})Xie, Dai, Hovy, Luong, and
  Le]{xie2020unsupervised}
Qizhe Xie, Zihang Dai, Eduard Hovy, Thang Luong, and Quoc Le.
\newblock Unsupervised data augmentation for consistency training.
\newblock \emph{Advances in Neural Information Processing Systems},
  33:\penalty0 6256--6268, 2020{\natexlab{a}}.

\bibitem[Xie et~al.(2020{\natexlab{b}})Xie, Luong, Hovy, and Le]{xie2020self}
Qizhe Xie, Minh-Thang Luong, Eduard Hovy, and Quoc~V Le.
\newblock Self-training with noisy student improves imagenet classification.
\newblock In \emph{Proceedings of the IEEE/CVF conference on computer vision
  and pattern recognition}, pages 10687--10698, 2020{\natexlab{b}}.

\bibitem[Zadrozny et~al.(2003)Zadrozny, Langford, and Abe]{zadrozny2003cost}
Bianca Zadrozny, John Langford, and Naoki Abe.
\newblock Cost-sensitive learning by cost-proportionate example weighting.
\newblock In \emph{Third IEEE international conference on data mining}, pages
  435--442. IEEE, 2003.

\bibitem[Zagoruyko and Komodakis(2016)]{BMVC2016_87}
Sergey Zagoruyko and Nikos Komodakis.
\newblock Wide residual networks.
\newblock In Edwin R.~Hancock Richard C.~Wilson and William A.~P. Smith,
  editors, \emph{Proceedings of the British Machine Vision Conference (BMVC)},
  pages 87.1--87.12. BMVA Press, September 2016.
\newblock ISBN 1-901725-59-6.
\newblock \doi{10.5244/C.30.87}.

\bibitem[Zhang et~al.(2021{\natexlab{a}})Zhang, Wang, Hou, WU, Wang, Okumura,
  and Shinozaki]{NEURIPS2021_995693c1}
Bowen Zhang, Yidong Wang, Wenxin Hou, HAO WU, Jindong Wang, Manabu Okumura, and
  Takahiro Shinozaki.
\newblock Flexmatch: Boosting semi-supervised learning with curriculum pseudo
  labeling.
\newblock In M.~Ranzato, A.~Beygelzimer, Y.~Dauphin, P.S. Liang, and J.~Wortman
  Vaughan, editors, \emph{Advances in Neural Information Processing Systems},
  volume~34, pages 18408--18419. Curran Associates, Inc., 2021{\natexlab{a}}.

\bibitem[Zhang et~al.(2021{\natexlab{b}})Zhang, Wang, Hou, Wu, Wang, Okumura,
  and Shinozaki]{zhang2021flexmatch}
Bowen Zhang, Yidong Wang, Wenxin Hou, Hao Wu, Jindong Wang, Manabu Okumura, and
  Takahiro Shinozaki.
\newblock Flexmatch: Boosting semi-supervised learning with curriculum pseudo
  labeling.
\newblock \emph{Advances in Neural Information Processing Systems}, 34,
  2021{\natexlab{b}}.

\bibitem[Zou et~al.(2019)Zou, Yu, Liu, Kumar, and Wang]{zou2019confidence}
Yang Zou, Zhiding Yu, Xiaofeng Liu, BVK Kumar, and Jinsong Wang.
\newblock Confidence regularized self-training.
\newblock In \emph{Proceedings of the IEEE/CVF International Conference on
  Computer Vision}, pages 5982--5991, 2019.

\end{thebibliography}

\end{document}